\newif\ifabstract
\abstracttrue
 \abstractfalse 
\newif\iffull
\ifabstract \fullfalse \else \fulltrue \fi

\documentclass[11pt]{article}
\usepackage{amsfonts}
\usepackage{amssymb,bm}
\usepackage{amstext}
\usepackage{wrapfig}
\usepackage{booktabs} 
\usepackage{amsmath}
\usepackage{xspace}
\usepackage{theorem}
\usepackage{graphicx}
\usepackage{url}
\usepackage{graphics}
\usepackage{colordvi}
\usepackage{colordvi}
\usepackage{subfigure}
\usepackage{xr-hyper}
\usepackage{hyperref}

\advance \oddsidemargin by -0.8in
\newcommand{\myparskip}{3pt}
\parskip \myparskip
\newcommand{\sgn}{\text{sgn}}

\newcommand{\e}{\varepsilon}

\newtheorem{theorem}{Theorem}
\newtheorem{lemma}{Lemma}

\newtheorem{corollary}{Corollary}
\newtheorem{claim}{Claim}
\newtheorem{proposition}{Proposition}
\newtheorem{assumption}{Assumption}
\newtheorem{definition}{Definition}

\newenvironment{proof}{\par \smallskip{\bf Proof:}}{\hfill\stopproof}
\def\stopproof{\square}
\def\square{\vbox{\hrule height.2pt\hbox{\vrule width.2pt height5pt \kern5pt
\vrule width.2pt} \hrule height.2pt}}
\newcommand{\footremember}[2]{%
   \footnote{#2}
    \newcounter{#1}
    \setcounter{#1}{\value{footnote}}%
}
\newcommand{\footrecall}[1]{%
    \footnotemark[\value{#1}]%
}

\begin{document}

\title{Adding One Neuron Can Eliminate All Bad Local Minima}
\author{Shiyu Liang\footremember{uiuc}{University of Illinois at Urbana-Champaign}\\ sliang26@illinois.edu
            \and Ruoyu Sun\footrecall{uiuc}\\ ruoyus@illinois.edu
            \and Jason D. Lee\footremember{usc}{ University of Southern California}\\jasonlee@marshall.usc.edu
            \and R. Srikant\footrecall{uiuc}\\rsrikant@illinois.edu
            }

\date{}

\maketitle

\begin{abstract}
One of the main difficulties in analyzing neural networks is the non-convexity of the loss function which may have many bad local minima.
 In this paper, we study the landscape of neural networks for binary classification tasks. Under mild assumptions, we prove that after adding one special neuron with a skip connection to the output, or one special neuron per layer, every local minimum is a global minimum.

\end{abstract}

\section{Introduction}
Deep neural networks have recently achieved huge success in various machine learning tasks (see,  \cite{krizhevsky2012imagenet}; \cite{goodfellow2013maxout}; \cite{wan2013regularization}, for example). However, a theoretical understanding of neural networks is largely lacking.
One of the difficulties in analyzing neural networks is the non-convexity of the loss function which allows the existence of many local minima with large losses. This was long considered a bottleneck of neural networks, and one of the reasons why convex formulations such as support vector machine~\cite{cortes1995support} were preferred previously. Given the recent empirical success of the deep neural networks, an interesting question is whether the non-convexity of the neural network is really an issue.

It has been widely conjectured that all local minima of the empirical loss lead to similar training performance~\cite{lecun2015deep,choromanska2015loss}. For example, prior works empirically showed that neural networks with identical architectures but different initialization points can converge to local minima with similar classification performance~\cite{krizhevsky2012imagenet, he2016deep, huang2017densely}. 
On the theoretical side, there have been many recent attempts to analyze the landscape of the neural network loss functions. A few works have studied deep networks, but they either require linear activation functions ~\cite{baldi1989neural, kawaguchi2016deep, freeman2016topology, hardt2016identity, yun2017global}, or require assumptions such as independence of ReLU activations  \cite{choromanska2015loss} and significant overparametrization  ~\cite{nguyen2017loss1,nguyen2017loss2,livni2014computational}.
There is a large body of works that study single-hidden-layer neural networks and provide various conditions under which a local search algorithm can find a global minimum \cite{du2018power,ge2017learning,andoni2014learning, sedghi2014provable, janzamin2015beating, haeffele2015global, gautier2016globally, brutzkus2017globally, soltanolkotabi2017learning, soudry2017exponentially, goel2017learning, du2017convolutional, zhong2017recovery, li2017convergence, liang2018understanding}. Note that even for single-layer networks, strong assumptions such as over-parameterization, very special neuron activation functions,  fixed second layer parameters and/or Gaussian data distribution are  often needed in the existing works. 
The presence of various strong assumptions reflects the difficulty of the problem: even for the single-hidden-layer nonlinear neural network, it seems hard to analyze the landscape, so it is reasonable to make various assumptions.

In addition to strong assumptions, the conclusions in many existing works do not apply to all local minima. One typical conclusion is about the local geometry, i.e., in a small neighborhood of the global minima no bad local minima exist \cite{zhong2017recovery, du2017convolutional, li2017convergence}. Another typical conclusion is that a subset of local minima are global minima \cite{haeffele2014structured, haeffele2015global, soudry2016no, nguyen2017loss1,nguyen2017loss2}. 
\cite{shamir2018resnets} has shown that a subset of second-order local minima can perform nearly as well as linear predictors. 
The presence of various conclusions also reflects the difficulty of the problem: while analyzing the global landscape seems hard, we may step back and analyze the local landscape or a ``majority'' of the landscape. 
Based on the above discussions,  an ideal theoretical result would state that with mild assumptions on the dataset, neural architectures and loss functions, every local minimum is a global minimum; existing results often make more than one strong assumption and/or prove weaker conclusions on the landscape. 
\subsection{Our Contributions}
Given this context, our main result is quite surprising: for binary classification, with a small modification of the neural architecture,  every local minimum is a global minimum of the loss function. Our result requires no assumption on the network size, the specific type of the original neural network, etc., yet our result applies to every local minimum. 
The major trick is adding one special neuron (with a skip connection) and an associated regularizer of this neuron.  
Our major result and its implications are as follows:
\begin{itemize}
	\item We focus on the binary classification problem with a smooth hinge loss function. 
	We prove the following result: for any neural network, by adding a special neuron (e.g., exponential neuron) to the network and adding a quadratic regularizer of this neuron, the new loss function has no bad local minimum. In addition, every local minimum achieves the minimum misclassification error. 
	 \item In the main result, the augmented neuron can be viewed as a skip connection from the input to the output layer. However, this skip connection is not critical, as the same result also holds if we add one special neuron to each layer of a fully-connected feedforward neural network.
	\item To our knowledge, this is the first result that no spurious local minimum exists for a wide class of deep nonlinear networks. Our result indicates that the class of ``good neural networks'' (neural networks such that there is an associated loss function with no spurious local minima) contains any network with one special neuron, thus this class is rather ``dense'' in the class of all neural networks: the distance between any neural network and a good neural network is just a neuron away.
\end{itemize}

The outline of the paper is as follows.  In Section~\ref{sec::prelim}, we present several notations.  In Section~\ref{sec::main-results}, we present the main result and several extensions on the main results are presented in Section~\ref{sec::extensions}.  We present the proof idea of the main result in Section~\ref{sec::proof-idea} and conclude this paper in Section~\ref{sec::conclusions}. All proofs are presented in Appendix.
	
\vspace{-0.05cm}
\section{Preliminaries}\label{sec::prelim}\vspace{-0.05cm}
\textbf{Feed-forward networks.} Given an input vector of dimension $d$, we consider a neural network with $L$ layers of neurons for binary classification. We denote by $M_{l}$ the number of neurons in the $l$-th layer (note that $M_{0}=d$). We denote the neural activation function by $\sigma$. Let $\bm{W}_{l}\in\mathbb{R}^{M_{l-1}\times M_{l}}$ denote the weight matrix connecting the $(l-1)$-th and $l$-th layer and $\bm{b}_{l}$ denote the bias vector for neurons in the $l$-th layer. Let $\bm{W}_{L+1}\in\mathbb{R}^{M_{L}}$ and ${b}_{L}\in\mathbb{R}$ denote the weight vector and bias scalar in the output layer, respectively. Therefore, the output of the network $f:\mathbb{R}^{d}\rightarrow\mathbb{R}$ can be expressed by\vspace{-0.2cm}
\begin{equation}\label{eq::fnn}
f(x;\bm{\theta})=\bm{W}_{L+1}^{\top}\bm{\sigma}\left(\bm{W}_{L}\bm{\sigma}\left(...\bm{\sigma}\left(\bm{W}_{1}^{\top}x+\bm{b}_{1}\right)+\bm{b}_{L-1}\right)+\bm{b}_{L}\right)+b_{L+1}.\end{equation}
\textbf{Loss and error}. We use $\mathcal{D}=\{(x_{i},y_{i})\}_{i=1}^{n}$ to denote a dataset containing $n$ samples, where $x_{i}\in\mathbb{R}^{d}$ and $y_{i}\in\{-1,1\}$ denote the feature vector and the label of the $i$-th sample, respectively.   Given a neural network $f(x;\bm{\theta})$ parameterized by $\bm{\theta}$ and a loss function $\ell:\mathbb{R}\rightarrow\mathbb{R}$, in binary classification tasks, we define the empirical loss $L_{n}(\bm{\theta})$ as the average loss of the network $f$ on a sample in the dataset and define  
the {training error} (also called the {misclassification error}) $R_{n}(\bm{\theta};f)$ as the misclassification rate of the  network $f$ on the dataset $\mathcal{D}$, i.e.,\vspace{-0.1cm}
\begin{equation}\label{eq::loss-error}
L_{n}(\bm{\theta})=\sum_{i=1}^{n}\ell(-y_{i}f(x_{i};\bm{\theta}))\quad\text{and}\quad {R}_{n}(\bm{\theta};f)=\frac{1}{n}\sum_{i=1}^{n}\mathbb{I}\{y_{i}\neq\sgn(f(x_{i};\bm{\theta}))\}.
\end{equation}
where $\mathbb{I}$ is the indicator function.  

\textbf{Tensors products.} We use $\bm{a}\otimes\bm{b}$ to denote the tensor product of vectors $\bm{a}$ and $\bm{b}$ and use $\bm{a}^{\otimes k}$ to denote the tensor product $\bm{a}\otimes ...\otimes\bm{a}$ where $\bm{a}$ appears $k$ times. For an $N$-th order tensor $\bm{T}\in\mathbb{R}^{d_{1}\times d_{2}\times...\times d_{N}}$ and $N$ vectors $\bm{u}_{1}\in\mathbb{R}^{d_{1}},\bm{u}_{2}\in\mathbb{R}^{d_{2}},...,\bm{u}_{N}\in\mathbb{R}^{d_{N}}$, we define 
$$\bm{T}\otimes\bm{u}_{1}...\otimes\bm{u}_{N}=\sum_{i_{1}\in[d_{1}],...,i_{N}\in[d_{N}]}\bm{T}(i_{1},...,i_{N})\bm{u}_{1}(i_{1})...\bm{u}_{N}(i_{N}),$$
where we use  $\bm{T}(i_{1},...,i_{N})$ to denote the $(i_{1},...,i_{N})$-th component of the tensor $\bm{T}$, $\bm{u}_{k}(i_{k})$ to denote the $i_{k}$-th component of the vector $\bm{u}_{k}$, $k=1,...,N$ and $[d_{k}]$ to denote the set $\{1,...,d_{k}\}$.

\section{Main Result}\label{sec::main-results}
In this section, we first present several important conditions on the loss function and the dataset in order to derive the main results. After that, we will present the main results. 

\subsection{Assumptions}\label{sec::assump}
In this subsection, we introduce two assumptions on the loss function and the dataset. 
\begin{assumption}[Loss function]\label{assump::loss}\vspace{-0.2cm}
 Assume that the loss function $\ell:\mathbb{R}\rightarrow\mathbb{R}$ is monotonically non-decreasing and twice differentiable, i.e., $\ell\in C^{2}$. Assume that every critical point of the loss function $\ell(z)$ is also a global minimum and every global minimum $z$ satisfies $z<0$.
 \vspace{-0.1cm}
\end{assumption}
A simple example of the loss function satisfying Assumption~\ref{assump::loss} is the polynomial hinge loss, i.e., $\ell(z)=[\max\{z+1,0\}]^{p}$, $p\ge 3$. It is always zero for $z \leq -1$ and behaves like a polynomial function in the region $z > -1$. Note that the condition that every global minimum of the loss function $\ell(z)$ is negative is not needed to prove the result that every local minimum of the empirical loss is globally minimal, but is necessary to prove that the global minimizer of  the empirical loss is also the minimizer of the misclassification rate.

\begin{assumption}[Realizability]\label{assump::realizability}\vspace{-0.2cm}
 Assume that there exists a set of parameters $\bm{\theta}$ such that the neural network $f(\cdot;\bm{\theta})$ is able to correctly classify all samples in the dataset $\mathcal{D}$. 
 \vspace{-0.1cm}
\end{assumption}
 By Assumption~\ref{assump::realizability}, we assume that the dataset is realizable by the neural architecture $f$. We note that this assumption is consistent with previous empirical observations~\cite{zhang2016understanding, krizhevsky2012imagenet, he2016deep} showing that at the end of the training process, neural networks usually achieve zero misclassification rates on the training sets.  However, as we will show later, if the loss function $\ell$ is convex, then we can prove the main result even without  Assumption~\ref{assump::realizability}.

\subsection{Main Result}\label{sec::single-exp}
In this subsection, we first introduce several  notations and next present the main result of the paper.   

Given a neural architecture $f(\cdot;\bm{\theta})$ defined on a $d$-dimensional Euclidean space and parameterized by a set of parameters $\bm{\theta}$, we define a new architecture $\tilde{f}$ by adding the output of an exponential neuron to the output of the network $f$, i.e., 
\begin{equation}\label{eq::single-exp}
\tilde{f}(x, \tilde{\bm{\theta}})=f(x;\bm{\theta})+a\exp\left(\bm{w}^{\top}x+b\right),
\end{equation}
where the vector $\tilde{\bm{\theta}}=(\bm{\theta},a,\bm{w}, b)$ denote the parametrization of the network $\tilde{f}$.
For this designed model, we define the empirical loss function as follows,
\begin{equation}\label{eq::loss-single}
\tilde{L}_{n}(\tilde{\bm{\theta}})=\sum_{i=1}^{n}\ell\left(-y_{i}\tilde{f}(x;\tilde{\bm{\theta}})\right)+\frac{\lambda a^{2}}{2},\end{equation}
where the scalar $\lambda$ is a positive real number, i.e., $\lambda>0$. Different from the empirical loss function ${L}_{n}$, the loss $\tilde{L}_{n}$ has an additional regularizer on the parameter $a$, since we aim to eliminate the impact of the exponential neuron on the output of the network $\tilde{f}$ at every local minimum of $\tilde{L}_{n}$. As we will show later, the exponential neuron is inactive at every local minimum of the empirical loss $\tilde{L}_{n}$. 
Now we  present the following theorem to show that every local minimum of the loss function $\tilde{L}_{n}$ is also a global minimum. 

\textbf{Remark:} Instead of viewing the exponential term in Equation~\eqref{eq::single-exp} as a neuron, one can also equivalently think of modifying the loss function to be 
$$\tilde{L}_{n}(\tilde{\bm{\theta}})=\sum_{i=1}^{n}\ell\left(-y_{i}(f(x_{i};\bm{\theta})+a\exp(\bm{w}^{\top}x_{i}+b))\right)+\frac{\lambda a^{2}}{2}.$$
Then, one can interpret Equation~\eqref{eq::single-exp} and~\eqref{eq::loss-single} as maintaining the original neural architecture and slightly modifying the loss function.

\vspace{-0.2cm}
\begin{theorem}\label{thm::single-exp}
Suppose that Assumption~\ref{assump::loss} and \ref{assump::realizability} hold. Assume that $\tilde{\bm{\theta}}^{*}=(\bm{\theta}^{*},a^{*},\bm{w}^{*}, b^{*})$ is a local minimum of the empirical loss function $\tilde{L}_{n}(\tilde{\bm{\theta}})$, then $\tilde{\bm{\theta}}^{*}$ is a global minimum of $\tilde{L}_{n}(\tilde{\bm{\theta}})$.  Furthermore,  $\bm{\theta}^{*}$ achieves the minimum loss value and the minimum misclassification rate on the dataset $\mathcal{D}$, i.e., $\bm{\theta}^{*}\in\arg\min_{\bm{\theta}}{L}_{n}(\bm{\theta})$ and $\bm{\theta}^{*}\in\arg\min_{\bm{\theta}}{R}_{n}(\bm{\theta};f)$.
\vspace{-0.2cm}
\end{theorem}
\textbf{Remarks:} (i) Theorem~\ref{thm::single-exp} shows that every local minimum $\tilde{\bm{\theta}}^{*}$ of the empirical loss $\tilde{L}_{n}$ is also a global minimum and shows that  $\bm{\theta}^{*}$ achieves the minimum training error and the minimum loss value on the original loss function $L_{n}$ at the same time. 
(ii) Since we do not require the explicit form of the neural architecture $f$, Theorem~\ref{thm::single-exp} applies to the neural architectures widely used in practice such as convolutional neural network~\cite{krizhevsky2012imagenet}, deep residual networks~\cite{he2016deep}, etc. This further indicates that the result holds for any real neural activation functions such as rectified linear unit (ReLU), leaky rectified linear unit (Leaky ReLU), etc. (iii) As we will show in the following corollary, at every local minimum $\tilde{\bm{\theta}}^{*}$, the exponential neuron is inactive. Therefore, at every local minimum $\tilde{\bm{\theta}}^{*}=(\bm{\theta}^{*},a ^{*},\bm{w}^{*},b^{*})$, the neural network $\tilde{f}$ with an augmented exponential neuron  is equivalent to the original neural network $f$.  
\begin{corollary}\label{cor::single-exp}\vspace{-0.2cm}
Under the conditions of Theorem~\ref{thm::single-exp}, if $\tilde{\bm{\theta}}^{*}=(\bm{\theta}^{*},a^{*},\bm{w}^{*}, b^{*})$ is a local minimum of the empirical loss function $\tilde{L}_{n}(\tilde{\bm{\theta}})$, then two neural networks $f(\cdot;{\bm{\theta}}^{*})$ and $\tilde{f}(\cdot;\tilde{\bm{\theta}}^{*})$ are equivalent, i.e., $f(x;{\bm{\theta}}^{*})=\tilde{f}(x;\tilde{\bm{\theta}}^{*})$, $\forall x\in\mathbb{R}^{d}$.
\vspace{-0.2cm}
\end{corollary}
Corollary~\ref{cor::single-exp} shows that at every local minimum, the exponential neuron does not contribute to the output of the neural network $\tilde{f}$. However, this does not imply that the exponential neuron is unnecessary, since several previous results~\cite{safran2017spurious,liang2018understanding} have already shown that the loss surface of pure ReLU neural networks are guaranteed to have bad local minima. Furthermore, to prove the main result under any dataset, the regularizer is also necessary, since \cite{liang2018understanding} has already shown that even with an augmented exponential neuron, the empirical loss without the regularizer still have bad local minima under some datasets. 

\section{Extensions}\label{sec::extensions}

\subsection{Eliminating the Skip Connection}\label{sec::fnn}

\begin{wrapfigure}{R}{0.45\linewidth}
\centering
 \vspace{-0.6cm}
  \includegraphics[width=1\linewidth]{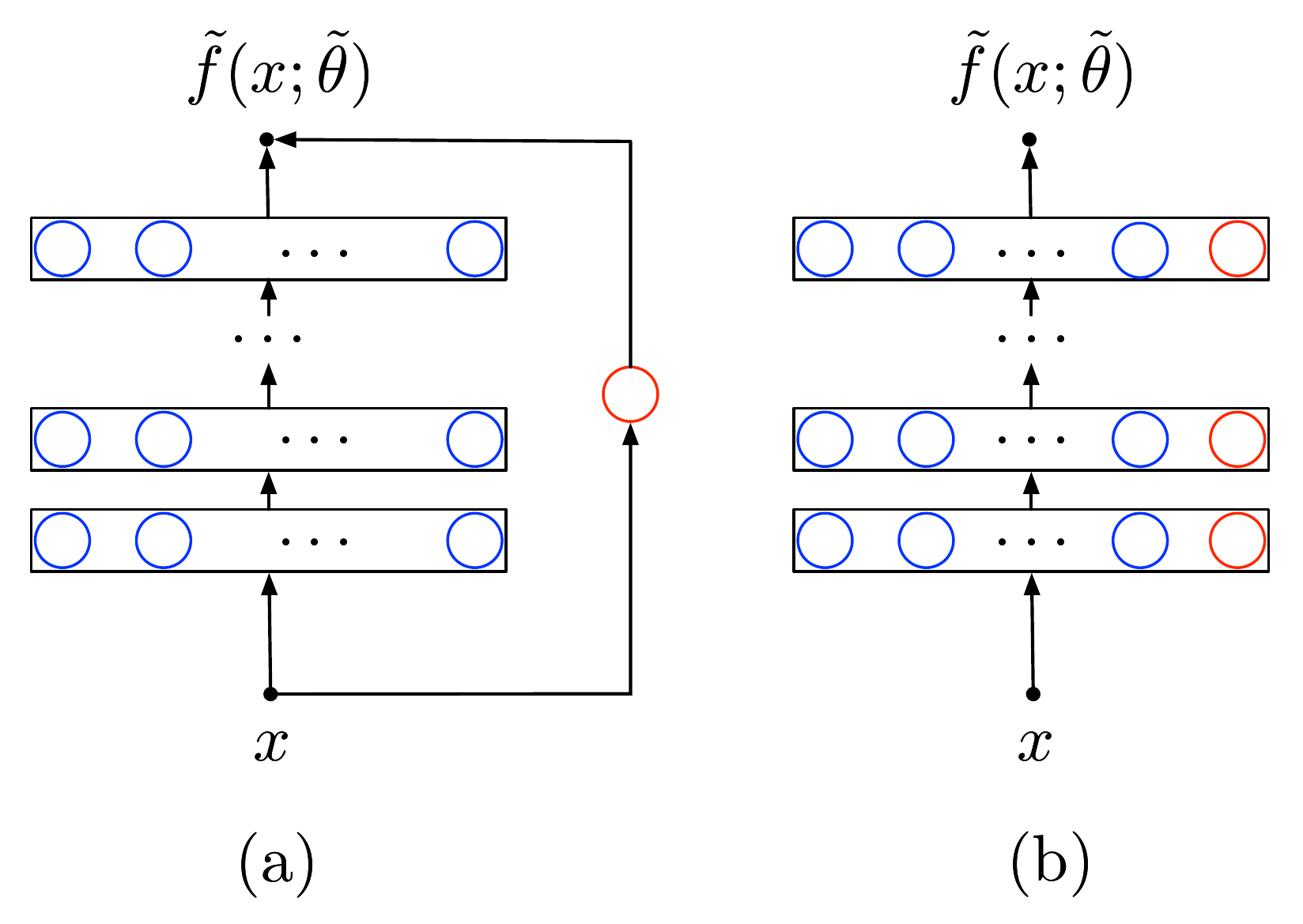}
  \vspace{-0.6cm}
  \caption{\small (a) The neural architecture considered in Theorem~\ref{thm::single-exp}. (b) The neural architecture considered in Theorem~\ref{thm::multi-exp}. The blue and red circles denote the neurons $\sigma$ in the original network and the augmented exponential neurons, respectively.}
  \vspace{-0.3cm}
  \label{fig::shortcut}
\end{wrapfigure}

As noted in the previous section, the exponential term in Equation~\eqref{eq::single-exp} can be viewed as a skip connection or a modification to the loss function. Our analysis also works under other architectures as well.
When the exponential term is viewed as a skip connection, the network architecture is as shown in Fig.~\ref{fig::shortcut}(a). 
This architecture is different from the canonical feedforward neural architectures as there is a direct path from the input layer to the output layer. 
In this subsection, we will show that the main result still holds if the model $\tilde{f}$ is defined as a
 feedforward neural network shown in Fig.~\ref{fig::shortcut}(b), where each layer of the network $f$ is augmented by an additional exponential neuron. 
 This is a standard fully connected neural network except for one special neuron at each layer.

\textbf{Notations.}  Given a fully-connected feedforward neural network $f(\cdot;\bm{\theta})$ defined by Equation~\eqref{eq::fnn}, we define a new fully connected feedforward neural network  $\tilde{f}$ by adding an additional exponential neuron to  each layer of the network $f$. We use the vector $\tilde{\bm{\theta}}=(\bm{\theta},\bm{\theta}_{\exp})$ to denote the parameterization of the network $\tilde{f}$, where $\bm{\theta}_{\exp}$ denotes the vector consisting of all augmented weights and biases. Let $\tilde{\bm{W}}_{l}\in\mathbb{R}^{(M_{l-1}+1)\times(M_{l}+1)}$ and $\tilde{\bm{b}}_{l}\in\mathbb{R}^{M_{l}+1}$ denote the weight matrix and the bias vector in the $l$-th layer of the network $\tilde{f}$, respectively. Let $\tilde{\bm{W}}_{L+1}\in\mathbb{R}^{(M_{L}+1)}$ and $\tilde{{b}}_{L+1}\in\mathbb{R}$ denote the weight vector and the bias scalar in the output layer of the network $\tilde{f}$, respectively. Without the loss of generality, we assume that the $(M_{l}+1)$-th neuron in the $l$-th layer is the augmented exponential neuron. 
Thus, the output of the network $\tilde{f}$ is expressed by 
\begin{equation}
\tilde{f}(x;\bm{\theta})=\tilde{\bm{W}}_{L+1}^{\top}\tilde{\bm{\sigma}}_{L+1}\left(\tilde{\bm{W}}_{L}\tilde{\bm{\sigma}}_{L}\left(...\tilde{\bm{\sigma}}_{1}\left(\tilde{\bm{W}}_{1}^{\top}x+\tilde{\bm{b}}_{1}\right)+\tilde{\bm{b}}_{L-1}\right)+\tilde{\bm{b}}_{L}\right)+\tilde{b}_{L+1},
\end{equation}
where $\tilde{\bm{\sigma}}_{l}:\mathbb{R}^{M_{l-1}+1}\rightarrow\mathbb{R}^{M_{l}+1}$ is a vector-valued activation function with the first $M_{l}$ components being the activation functions $\sigma$ in the network $f$ and with the last component being the exponential function, i.e., $\tilde{\bm{\sigma}}_{l}(z)=(\sigma(z),...,\sigma(z),\exp(z))$. Furthermore, we use the $\tilde{\bm{w}}_{l}$ to denote the vector  in the $(M_{l-1}+1)$-th row of the matrix $\tilde{\bm{W}}_{l}$. In other words, the components of the vector $\tilde{\bm{w}}_{l}$ are the weights on the edges connecting the exponential neuron in the $(l-1)$-th layer and the neurons in the $l$-th layer. 
For this feedforward network, we define an empirical loss function as \vspace{-0.1cm}
\begin{equation}
\tilde{L}_{n}(\tilde{\bm{\theta}})=\sum_{i=1}^{n}\ell(-y_{i}\tilde{f}(x_{i};\tilde{\bm{\theta}}))+\frac{\lambda}{2}\sum_{l=2}^{L+1}\left\|\tilde{\bm{w}}_{l}\right\|_{2L}^{2L}
\end{equation}
where  $\|\bm{a}\|_{p}$ denotes the $p$-norm of a vector $\bm{a}$ and $\lambda$ is a positive real number, i.e., $\lambda>0$. Similar to the empirical loss  discussed in the previous section, we add a regularizer to eliminate the impacts of all exponential neurons on the output of the network. Similarly, we can prove that at every local minimum of $\tilde{L}_{n}$, all exponential neurons are inactive. Now we present the following theorem to show that if the set of parameters $\tilde{\bm{\theta}}^{*}=(\bm{\theta}^{*},\bm{\theta}^{*}_{\exp})$ is a local minimum of the empirical loss function $\tilde{L}_{n}(\tilde{\bm{\theta}})$, then $\tilde{\bm{\theta}}^{*}$ is a global minimum and
$\bm{\theta}^{*}$ is a  global minimum of both minimization  problems $\min_{\bm{\theta}}L_{n}(\bm{\theta})$ and $\min_{\bm{\theta}}R_{n}(\bm{\theta};f)$. This means that the neural network $f(\cdot;\bm{\theta}^{*})$ simultaneously achieves the globally minimal loss value and misclassification rate on the dataset $\mathcal{D}$. 
\begin{theorem}\label{thm::multi-exp}\vspace{-0.2cm}
Suppose that Assumption~\ref{assump::loss} and \ref{assump::realizability} hold. Suppose that the activation function $\sigma$ is differentiable.  Assume that $\tilde{\bm{\theta}}^{*}=(\bm{\theta}^{*},\bm{\theta}_{\exp}^{*})$ is a local minimum of the empirical loss function $\tilde{L}_{n}(\tilde{\bm{\theta}})$, then $\tilde{\bm{\theta}}^{*}$ is a global minimum of $\tilde{L}_{n}(\tilde{\bm{\theta}})$.  Furthermore, $\bm{\theta}^{*}$ achieves the minimum loss value and the minimum misclassification rate on the dataset $\mathcal{D}$, i.e., $\bm{\theta}^{*}\in\arg\min_{\bm{\theta}}{L}_{n}(\bm{\theta})$ and $\bm{\theta}^{*}\in\arg\min_{\bm{\theta}}{R}_{n}(\bm{\theta};f)$.
\vspace{-0.2cm}
\end{theorem}
\textbf{Remarks}: (i) This theorem is not a direct corollary of the result in the previous section, but the proof ideas are similar. 
 (ii) Due to the assumption on the  differentiability of the activation function $\sigma$, Theorem~\ref{thm::multi-exp} does not apply to the neural networks consisting of non-smooth neurons such as ReLUs, Leaky ReLUs, etc. 
  (iii) Similar to Corollary~\ref{cor::single-exp}, we will present the following corollary to show that at every local minimum $\tilde{\bm{\theta}}^{*}=(\bm{\theta}^{*},\bm{\theta}^{*}_{\exp})$, the neural network $\tilde{f}$ with augmented exponential neurons is equivalent to the original neural network ${f}$. 
\begin{corollary}\label{cor::multi-exp}\vspace{-0.2cm}
Under the conditions in Theorem~\ref{thm::multi-exp}, if $\tilde{\bm{\theta}}^{*}=(\bm{\theta}^{*},\bm{\theta}_{\exp}^{*})$ is a local minimum of the empirical loss function $\tilde{L}_{n}(\tilde{\bm{\theta}})$, then two neural networks $f(\cdot;{\bm{\theta}}^{*})$ and $\tilde{f}(\cdot;\tilde{\bm{\theta}}^{*})$ are equivalent, i.e., $f(x;{\bm{\theta}}^{*})=\tilde{f}(x;\tilde{\bm{\theta}}^{*}),\forall x\in\mathbb{R}^{d}$.
\vspace{-0.2cm}
\end{corollary}
Corollary~\ref{cor::multi-exp} further shows that even if we add an exponential neuron to each layer of the original network $f$, at every local minimum of the empirical loss,  all exponential neurons are inactive.

\subsection{Neurons}\label{sec::neuron}
In this subsection, we will show that even if the exponential neuron is replaced by a monomial neuron, the main result still holds under additional assumptions. Similar to the case where exponential neurons are used, 
given a neural network $f(x;\bm{\theta})$, we define a new neural network $\tilde{f}$ by adding the output of a monomial neuron of degree $p$ to the output of the original model $f$, i.e.,  
\begin{equation}\tilde{f}(x;\tilde{\bm{\theta}})=f(x;\bm{\theta})+a\left(\bm{w}^{\top}x+b\right)^{p}.\end{equation}
In addition, the empirical loss function $\tilde{L}_{n}$ is exactly the same as the loss function defined by Equation~\eqref{eq::loss-single}. 
Next, we will present the following theorem to show that if all samples in the dataset $\mathcal{D}$ can be correctly classified by a polynomial of degree $t$ and the degree of the augmented monomial is not smaller than $t$ (i.e., $p\ge t$), then every local minimum of the empirical loss function $\tilde{L}_{n}(\tilde{\bm{\theta}})$ is also a global minimum. We note that the degree of a monomial is the sum of powers of all variables in this monomial and the degree of a polynomial is the maximum degree of its monomial. 

\begin{proposition}\label{thm::single-monomial}\vspace{-0.2cm}
Suppose that Assumptions~\ref{assump::loss} and~\ref{assump::realizability} hold. 
Assume that all samples in the dataset $\mathcal{D}$ can be correctly classified by a polynomial of degree $t$ and $p\ge t$.
 Assume that $\tilde{\bm{\theta}}^{*}=(\bm{\theta}^{*},a^{*},\bm{w}^{*}, b^{*})$ is a local minimum of the empirical loss function $\tilde{L}_{n}(\tilde{\bm{\theta}})$, then $\tilde{\bm{\theta}}^{*}$ is a global minimum of $\tilde{L}_{n}(\tilde{\bm{\theta}})$.  Furthermore, $\bm{\theta}^*$ is a global minimizer of both problems $\min_{\bm{\theta}}L_{n}(\bm{\theta})$ and $\min R_{n}(\bm{\theta};f)$.
\vspace{-0.2cm}
\end{proposition}
\textbf{Remarks:} (i) We note  that, similar to Theorem~\ref{thm::single-exp}, Proposition~\ref{thm::single-monomial} applies to all neural architectures and all neural activation functions defined on $\mathbb{R}$, as we do not require the explicit form of the neural network $f$. (ii) It follows from the Lagrangian interpolating polynomial and Assumption~\ref{assump::realizability} that for a dataset consisted of $n$ different samples, there always exists a polynomial $P$ of degree smaller $n$ such that the polynomial $P$ can correctly classify all points in the dataset. This indicates that Proposition~\ref{thm::single-monomial} always holds if $p\ge n$. (iii) Similar to Corollary~\ref{cor::single-exp} and \ref{cor::multi-exp}, we can show that at every local minimum $\tilde{\bm{\theta}}^{*}=(\bm{\theta}^{*}, a^{*},\bm{w}^{*},b^{*})$, the neural network $\tilde{f}$ with an augmented monomial neuron is equivalent to the original neural network $f$. 

\subsection{Allowing Random Labels}
In previous subsections, we assume the realizability of the dataset by the neural network which implies that the label of a given feature vector  is unique. It does not cover the case where the dataset contains two samples with the same feature vector but with different labels  (for example, the same image can be labeled differently by two different people). Clearly, in this case, no model can correctly classify all samples in this dataset. Another simple example of this case is the mixture of two Gaussians where the data samples are drawn from each of the two Gaussian distributions with certain probability.

In this subsection, we will show that under this broader setting that one feature vector may correspond to two different labels, with a slightly stronger assumption on the convexity of the loss $\ell$, the same result still holds. 
The formal statement is present by the following proposition. 
\begin{proposition}\label{thm::convex}
	\vspace{-0.2cm}
	Suppose that Assumption~\ref{assump::loss} holds and the loss function $\ell$ is convex. Assume that $\tilde{\bm{\theta}}^{*}=(\bm{\theta}^{*},a^{*},\bm{w}^{*}, b^{*})$ is a local minimum of the empirical loss function $\tilde{L}_{n}(\tilde{\bm{\theta}})$, then $\tilde{\bm{\theta}}^{*}$ is a global minimum of $\tilde{L}_{n}(\tilde{\bm{\theta}})$.  Furthermore,  $\bm{\theta}^{*}$ achieves the minimum loss value and the minimum misclassification rate on the dataset $\mathcal{D}$, i.e., $\bm{\theta}^{*}\in\arg\min_{\bm{\theta}}{L}_{n}(\bm{\theta})$ and $\bm{\theta}^{*}\in\arg\min_{\bm{\theta}}{R}_{n}(\bm{\theta};f)$.
	\vspace{-0.2cm}
\end{proposition}
\textbf{Remark:} The differences of Proposition~\ref{thm::convex} and Theorem~\ref{thm::single-exp} can be understood in the following ways.
First, as stated previously, Proposition~\ref{thm::convex} allows a feature vector to have two different labels, but Theorem ~\ref{thm::single-exp} does not.
Second, the minimum misclassification rate under the conditions in Theorem ~\ref{thm::single-exp} must be zero, while in Proposition~\ref{thm::convex}, the minimum misclassification rate can be nonzero. 



\subsection{High-order Stationary Points}\label{sec::stationary-points}
In this subsection, we characterize the high-order stationary points of the empirical loss $\tilde{L}_{n}$ shown in Section~\ref{sec::single-exp}. We first introduce the definition of the high-order stationary point and next show that every stationary point of the loss $\tilde{L}_{n}$ with a sufficiently high order is also a global minimum.  
\begin{definition}[$k$-th order stationary point]\label{def::saddle-point}\vspace{-0.2cm}
A critical point $\bm{\theta}_{0}$ of a function $L(\bm{\theta})$ is a $k$-th order stationary point, if there exists positive constant $C,\e>0$ such that for every $\bm{\theta}$ with $\|\bm{\theta}-\bm{\theta}_{0}\|_{2}\le \e$, $L(\bm{\theta})\ge L(\bm{\theta}_{0})-C\|\bm{\theta}-\bm{\theta}_{0}\|_{2}^{k+1}$.
\vspace{-0.2cm}
\end{definition}
Next, we will show that if a polynomial of degree $p$ can correctly classify all points in the dataset, then every stationary point of the order at least $2p$ is a global minimum and the set of parameters corresponding to this stationary point achieves the minimum training error.  
\begin{proposition}\label{thm::stationary}\vspace{-0.2cm}
Suppose that Assumptions~\ref{assump::loss} and~\ref{assump::realizability} hold. Assume that all samples in the dataset can be correctly classified by a polynomial of degree $p$.  
Assume that $\tilde{\bm{\theta}}^{*}=(\bm{\theta}^{*},a^{*},\bm{w}^{*}, b^{*})$ is a $k$-th order stationary point of the empirical loss function $\tilde{L}_{n}(\tilde{\bm{\theta}})$ and $k\ge2p$, then $\tilde{\bm{\theta}}^{*}$ is a global minimum of $\tilde{L}_{n}(\tilde{\bm{\theta}})$.  Furthermore, the neural network ${f}(\cdot;{\bm{\theta}}^{*})$ achieves the minimum misclassification rate  on the dataset $\mathcal{D}$, i.e., $\bm{\theta}^{*}\in\arg\min_{\bm{\theta}}{R}_{n}(\bm{\theta};f)$.
\vspace{-0.2cm}
\end{proposition}
One implication of Proposition~\ref{thm::stationary} is that  if a dataset is linearly separable, then every second order stationary point of the empirical loss function is a global minimum and, at this stationary point, the neural network achieves zero training error. 
When the dataset is not linearly separable, our result only covers fourth or higher order stationary point of the empirical loss. 


\section{Proof Idea}\label{sec::proof-idea}
In this section, we provide overviews of the proof of Theorem~\ref{thm::single-exp}. 

\subsection{Important Lemmas}
In this subsection, we present two important lemmas where the proof of Theorem~\ref{thm::single-exp} is based. 
\begin{lemma}	\label{lemma::a=0}
Under Assumption~\ref{assump::loss} and $\lambda>0$, if $\tilde{\bm{\theta}}^{*} = (\bm{\theta}^{*},a^{*},\bm{w}^{*},b^{*})$ is a local minimum of $\tilde{L}_{n}$,  then (i) $a^{*}=0$, (ii) for any  integer $p\ge0$, the following equation holds for all unit vector $\bm{u}:\|\bm{u}\|_{2}=1$,
\begin{equation}\label{lemma::eq-1}
\sum_{i=1}^n \ell'\left( -y_if(x_i;\bm{\theta}^{*})\right) y_i e^{{\bm{w}^{*}}^\top x_i +b^{*}} (\bm{u}^\top x_{i})^{p} =0.
\end{equation}
\end{lemma}

\begin{lemma}\label{lemma::tensor}
For any  integer $k\ge0$ and any sequence $\{c_{i}\}_{i=1}^{n}$, if  $\sum_{i=1}^{n}c_{i}(\bm{u}^{\top}x_{i})^{k}=0$ 
holds for all unit vector $\bm{u}:\|\bm{u}\|_{2}=1$, then the $k$-th order tensor $\bm{T}_{k}=\sum_{i=1}^{n}c_{i}x_{i}^{\otimes k}$ is a $k$-th order zero tensor.
\end{lemma}

\subsection{Proof Sketch of Lemma~\ref{lemma::a=0}}
\textbf{Proof sketch of Lemma~\ref{lemma::a=0}}($i$): To prove $a^{*}=0$, we only need to check the first order conditions of local minima. By assumption that $\tilde{\bm{\theta}}^{*}=(\bm{\theta}^{*}, a^{*},\bm{w}^{*},b^{*})$ is a local minimum of $\tilde{L}_{n}$, then the derivative of $\tilde{L}_{n}$ with respect to $a$ and $b$ at the point $\tilde{\bm{\theta}}^{*}$ are all zeros, i.e., 
	\begin{align}
	\left.\nabla_ a \tilde{L}_{n}(\tilde{\bm{\theta}})\right|_{\tilde{\bm{\theta}}=\tilde{\bm{\theta}}^{*}}  &=-\sum_{i=1}^n \ell'\left( -y_if(x_i;\bm{\theta}^{*})- y_ia^{*} e^{{\bm{w}^{*}}^\top x_i +b^{*}}\right) y_i  \exp({\bm{w}^{*}}^\top x_i +b^{*}) +\lambda a^{*}=0, \notag\\
	\left.\nabla_b \tilde{L}_{n}(\tilde{\bm{ \theta}})\right|_{\tilde{\bm{\theta}}=\tilde{\bm{\theta}}^{*}} &=- a^{*} \sum_{i=1}^n \ell'\left( -y_if(x_i;\bm{\theta}^{*})- y_ia^{*} e^{{\bm{w}^{*}}^\top x_i +b^{*}}\right) y_i  \exp({\bm{w}^{*}}^\top x_i+b^{*})=0. \notag
	\end{align}
From the above equations, it is not difficult to see that $a^{*}$ satisfies $\lambda {a^{*}}^{2}=0$ or, equivalently, $a^{*}=0$. 

We note that the main observation we are using here is that the derivative of the exponential neuron is itself. Therefore, it is not difficult to see that the same proof holds for all neuron activation function $\sigma$ satisfying $\sigma'(z)=c\sigma(z), \forall z\in\mathbb{R}$ for some constant $c$. In fact, with a small modification of the proof, we can show that the same proof works for all neuron activation functions satisfying $\sigma(z)=(c_{1}z+c_{0})\sigma'(z), \forall z\in\mathbb{R}$ for some constants $c_{0}$ and $c_{1}$.
This further indicates that the same proof  holds for the monomial neurons and thus the proof of Proposition~\ref{thm::single-monomial} follows directly from the proof of Theorem~\ref{thm::single-exp}. 

\textbf{Proof sketch of Lemma~\ref{lemma::a=0}}($ii$): The main idea of the proof is to use the high order information of the local minimum to derive Equation~\eqref{lemma::eq-1}. Due to the assumption that $\tilde{\bm{\theta}}=(\bm{\theta}^{*}, a^{*}, \bm{w}^{*}, b^{*})$ is a local minimum of the empirical loss function $\tilde{L}_{n}$, there exists a bounded local region such that the parameters $\tilde{\bm{\theta}}^{*}$ achieve the minimum loss value in this region, i.e.,  $\exists\delta\in(0,1)$ such that $\tilde{L}_{n}(\tilde{\bm{\theta}}^{*}+\bm{\Delta})\ge\tilde{L}_{n}(\tilde{\bm{\theta}}^{*})$ for $\forall\bm{\Delta}: \|\bm{\Delta}\|_{2}\le\delta$. 

Now, we use $\delta_{a}$, $\bm{\delta_{w}}$ to denote the perturbations on the parameters $a$ and $\bm{w}$, respectively. Next, we consider the loss value at the point $\tilde{\bm{\theta}}^{*}+\bm{\Delta}=(\bm{\theta}^{*}, a^{*}+\delta_{a}, \bm{w}^{*}+\bm{\delta_{w}}, b^{*})$, where we set $|\delta_{a}|=e^{-1/\e}$ and $\bm{\delta_{w}}=\e\bm{u}$ for an arbitrary unit vector $\bm{u}:\|\bm{u}\|_{2}=1$. Therefore, as $\e$ goes to zero, the perturbation magnitude $\|\bm{\Delta}\|_{2}$ also goes to zero and this indicates that there exists an $\e_{0}\in(0,1)$ such that  $\tilde{L}_{n}(\tilde{\bm{\theta}}^{*}+\bm{\Delta})\ge\tilde{L}_{n}(\tilde{\bm{\theta}}^{*})$ for $\forall \e\in[0, \e_{0})$. By the result $a^{*}=0$, shown in Lemma~\ref{lemma::a=0}($i$), the output of the model $\tilde{f}$ under parameters $\tilde{\bm{\theta}}^{*}+\bm{\Delta}$ can be expressed by 
$$\tilde{f}(x;\tilde{\bm{\theta}}^{*}+\bm{\Delta})=f(x;\bm{\theta}^{*})+\delta_{a}\exp(\bm{\delta_{w}}^{\top}x)\exp({\bm{w}^{*}}^{\top}x+b^{*}).$$
For simplicity of notation, let $g(x;\tilde{\bm{\theta}}^{*}, \bm{\delta_{w}})=\exp(\bm{\delta_{w}}^{\top}x)\exp({\bm{w}^{*}}^{\top}x+b^{*})$.  From the second order Taylor expansion with Lagrangian remainder and the assumption that $\ell$ is twice differentiable,  it follows that there exists a constant $C(\tilde{\bm{\theta}}^{*},\mathcal{D})$ depending only on the local minimizer $\tilde{\bm{\theta}}$ and the dataset $\mathcal{D}$ such that the following inequality holds for every sample in the dataset and every $\e\in[0,\e_{0})$, 
\begin{align*}
\ell(-y_{i}\tilde{f}(x_{i};\tilde{\bm{\theta}}^{*}+\bm{\Delta}))&\le\ell(-y_{i}f(x_{i};\bm{\theta}^{*}))+\ell'(-y_{i}f(x_{i};\bm{\theta}^{*}))(-y_{i})\delta_{a}g(x_{i};\tilde{\bm{\theta}}^{*}, \bm{\delta_{w}})+C(\tilde{\bm{\theta}}^{*},\mathcal{D})\delta_{a}^{2}.
\end{align*}
Summing the above inequality over all samples in the dataset and recalling that $\tilde{L}_{n}(\tilde{\bm{\theta}}^{*}+\bm{\Delta})\ge \tilde{L}_{n}(\tilde{\bm{\theta}}^{*})$ holds for all $\e\in[0,\e_{0})$, we obtain 
\begin{equation*}
       -\sgn(\delta_{a})\sum_{i=1}^{n}\ell'(-y_{i}f(x_{i};\bm{\theta}^{*}))y_{i}\exp(\varepsilon\bm{u}^{\top}x_{i})\exp({\bm{w}^{*}}^{\top}x_{i}+b^{*})+[nC(\tilde{\bm{\theta}}^{*},\mathcal{D})+\lambda/2]\exp(-1/\varepsilon)\ge 0.
\end{equation*}
Finally, we complete the proof by induction. Specifically, for the base hypothesis where $p=0$, we can take the limit on the both sides of the above inequality as $\e\rightarrow0$, using the property that $\delta_{a}$ can be either positive or negative  and thus establish the base case where $p=0$. For the higher order case, we  can first assume that Equation~\eqref{lemma::eq-1} holds for $p=0,...,k$ and then subtract these equations from the above inequality. After taking the limit on the both sides of the inequality as $\e\rightarrow0$, we can prove that Equation~\eqref{lemma::eq-1} holds for $p=k+1$. Therefore, by induction, we can prove that  Equation~\eqref{lemma::eq-1} holds for any non-negative integer $p$.

\subsection{Proof Sketch of Lemma~\ref{lemma::tensor}}
The proof of Lemma~\ref{lemma::tensor} follows directly from the results in reference~\cite{zhang2012best}. 
It is easy to check that, for every sequence $\{c_{i}\}_{i=1}^{n}$ and every non-negative integer $k\ge0$, the $k$-th order tensor $\bm{T}_{k}=\sum_{i=1}^{n}c_{i}x_{i}^{\otimes k}$ is a symmetric tensor. From Theorem~1 in \cite{zhang2012best}, it directly follows that  
$$\max_{\bm{u_{1},...,u_{k}}:\|\bm{u_{1}}\|_{2}=...=\|\bm{u_{k}}\|_{2}=1}|\bm{T}_{k}(\bm{u}_{1},...,\bm{u}_{k})|=\max_{\bm{u}:\|\bm{u}\|_{2}=1}|\bm{T}_{k}(\bm{u},...,\bm{u})|.$$
Furthermore, by assumption that $\bm{T}_{k}(\bm{u},...,\bm{u})=\sum_{i=1}^{n}c_{i}(\bm{u}^{\top}x_{i})^{k}=0$
holds for all $\|\bm{u}\|_{2}=1$, then 
$$\max_{\bm{u_{1},...,u_{k}}:\|\bm{u_{1}}\|_{2}=...=\|\bm{u_{k}}\|_{2}=1}|\bm{T}_{k}(\bm{u}_{1},...,\bm{u}_{k})|=0,$$
and this is equivalent to $\bm{T}_{k}=\bm{0}_{d}^{\otimes k},$ where $\bm{0}_{d}$ is the zero vector in the $d$-dimensional space. 

\subsection{Proof Sketch of Theorem~\ref{thm::single-exp}}
For every dataset $\mathcal{D}$ satisfying Assumption~\ref{assump::realizability}, by the Lagrangian interpolating polynomial, there always exists a polynomial  $P(x)=\sum_{j}c_{j}\pi_{j}(x)$ defined on $\mathbb{R}^{d}$ such that it can correctly classify all samples in the dataset with margin at least one, i.e., $y_{i}P(x_{i})\ge 1,\forall i\in[n]$, where $\pi_{j}$ denotes the $j$-th monomial in the polynomial $P(x)$. 
Therefore, from Lemma~\ref{lemma::a=0} and \ref{lemma::tensor}, it follows that 
\begin{align*}
\sum_{i=1}^{n}\ell'(-y_{i}f(x_{i};\bm{\theta}^{*}))e^{{\bm{w}^{*}}^{\top}x_{i}+b^{*}}y_{i}P(x_{i})=\sum_{j}c_{j}\sum_{i=1}^{n}\ell'(-y_{i}f(x_{i};\bm{\theta}^{*}))y_{i}e^{{\bm{w}^{*}}^{\top}x_{i}+b^{*}}\pi_{j}(x_{i})=0.
\end{align*}
Since $y_{i}P(x_{i})\ge 1$ and $e^{{\bm{w}^{*}}^{\top}x_{i}+b^{*}}>0$ hold for  $\forall i\in[n]$ and the loss function $\ell$ is a non-decreasing function, i.e., $\ell'(z)\ge 0,\forall z\in\mathbb{R}$, then $\ell'(-y_{i}f(x_{i};\bm{\theta}^{*}))=0$ holds for all $i\in[n]$. In addition, from the assumption that every critical point of the loss function $\ell$ is a global minimum, it follows that $z_{i}=-y_{i}f(x_{i};\bm{\theta}^{*})$ achieves the global minimum of the loss function $\ell$ and this further indicates that $\bm{\theta}^{*}$ is a global minimum of the empirical loss $L_{n}(\bm{\theta})$. Furthermore, since at every local minimum, the exponential neuron is inactive, $a^{*}=0$, then the set of parameters $\tilde{\bm{\theta}}^{*}$ is a global minimum of the loss function $\tilde{L}_{n}(\tilde{\bm{\theta}})$. Finally, since every critical point of the loss function $\ell(z)$ satisfies $z<0$, then for every sample, $\ell'(-y_{i}f(x_{i};\bm{\theta}^{*}))=0$ indicates that $y_{i}f(x_{i};\bm{\theta}^{*})>0$, or, equivalently, $y_{i}=\sgn(f(x_{i};\bm{\theta}^{*}))$. Therefore, the set of parameters $\bm{\theta}^{*}$ also minimizes the training error. In summary, the set of parameters $\tilde{\bm{\theta}}^{*}=(\bm{\theta}^{*},a^{*}, \bm{w}^{*}, b^{*})$ minimizes the  loss function $\tilde{L}_{n}(\tilde{\bm{\theta}})$ and the set of parameters $\bm{\theta}^{*}$ simultaneously minimizes the empirical loss function $L_{n}(\bm{\theta})$ and the training error $R_{n}(\bm{\theta};f)$. 
 
\section{Conclusions}\label{sec::conclusions}
One of the difficulties in analyzing neural networks is the non-convexity of the loss functions which allows the existence of many spurious minima with large loss values. In this paper, we prove that for any neural network, by adding a special neuron and  an associated regularizer, the new loss function has no spurious local minimum. In addition, we prove that, at every local minimum of this new loss function, the exponential neuron is inactive and this means that the augmented neuron and regularizer improve the landscape of the loss surface without affecting the representing power of the original neural network. 
We also extend the main result in a few ways. First, while adding a special neuron makes the network different from a classical neural network architecture, the same result also holds for a standard fully connected network  with one special neuron added to each layer.
Second, the same result holds if we change the exponential neuron to a polynomial neuron with a degree dependent on the data. Third, the same result holds even if one feature vector corresponds to both labels.

%


\bibliography{nips_2018}

\begin{thebibliography}{10}

\bibitem{krizhevsky2012imagenet}
A.~Krizhevsky, I.~Sutskever, and G.~E. Hinton.
\newblock Imagenet classification with deep convolutional neural networks.
\newblock In {\em NIPS}, 2012.

\bibitem{goodfellow2013maxout}
I.~J Goodfellow, D.~Warde-Farley, M.~Mirza, A.~Courville, and Y.~Bengio.
\newblock Maxout networks.
\newblock {\em arXiv preprint arXiv:1302.4389}, 2013.

\bibitem{wan2013regularization}
L.~Wan, M.~Zeiler, S.~Zhang, Y.~Le~Cun, and R.~Fergus.
\newblock Regularization of neural networks using dropconnect.
\newblock In {\em ICML}, pages 1058--1066, 2013.

\bibitem{cortes1995support}
C.~Cortes and V.~Vapnik.
\newblock Support-vector networks.
\newblock {\em Machine learning}, 1995.

\bibitem{lecun2015deep}
Y.~LeCun, Y.~Bengio, and G.~E. Hinton.
\newblock Deep learning.
\newblock {\em Nature}, 521(7553):436, 2015.

\bibitem{choromanska2015loss}
A.~Choromanska, M.~Henaff, M.~Mathieu, G.~Arous, and Y.~LeCun.
\newblock The loss surfaces of multilayer networks.
\newblock In {\em AISTATS}, 2015.

\bibitem{he2016deep}
K.~He, X.~Zhang, S.~Ren, and J.~Sun.
\newblock Deep residual learning for image recognition.
\newblock In {\em CVPR}, pages 770--778, 2016.

\bibitem{huang2017densely}
G.~Huang and Z.~Liu.
\newblock Densely connected convolutional networks.
\newblock In {\em CVPR}, 2017.

\bibitem{baldi1989neural}
P.~Baldi and K.~Hornik.
\newblock Neural networks and principal component analysis: Learning from
  examples without local minima.
\newblock {\em Neural networks}, 2(1):53--58, 1989.

\bibitem{kawaguchi2016deep}
K.~Kawaguchi.
\newblock Deep learning without poor local minima.
\newblock In {\em NIPS}, pages 586--594, 2016.

\bibitem{freeman2016topology}
C~D. Freeman and J.~Bruna.
\newblock Topology and geometry of half-rectified network optimization.
\newblock {\em arXiv preprint arXiv:1611.01540}, 2016.

\bibitem{hardt2016identity}
M.~Hardt and T.~Ma.
\newblock Identity matters in deep learning.
\newblock {\em ICLR}, 2017.

\bibitem{yun2017global}
C.~Yun, S.~Sra, and A.~Jadbabaie.
\newblock Global optimality conditions for deep neural networks.
\newblock {\em arXiv preprint arXiv:1707.02444}, 2017.

\bibitem{nguyen2017loss1}
Q.~Nguyen and M.~Hein.
\newblock The loss surface and expressivity of deep convolutional neural
  networks.
\newblock {\em arXiv preprint arXiv:1710.10928}, 2017.

\bibitem{nguyen2017loss2}
Q.~Nguyen and M.~Hein.
\newblock The loss surface and expressivity of deep convolutional neural
  networks.
\newblock {\em arXiv preprint arXiv:1710.10928}, 2017.

\bibitem{livni2014computational}
R.~Livni, S.~Shalev-Shwartz, and O.~Shamir.
\newblock On the computational efficiency of training neural networks.
\newblock In {\em NIPS}, 2014.

\bibitem{du2018power}
S.~S Du and J.~D Lee.
\newblock On the power of over-parametrization in neural networks with
  quadratic activation.
\newblock {\em arXiv preprint arXiv:1803.01206}, 2018.

\bibitem{ge2017learning}
R.~Ge, J.~D Lee, and T.~Ma.
\newblock Learning one-hidden-layer neural networks with landscape design.
\newblock {\em ICLR}, 2018.

\bibitem{andoni2014learning}
A.~Andoni, R.~Panigrahy, G.~Valiant, and L.~Zhang.
\newblock Learning polynomials with neural networks.
\newblock In {\em ICML}, 2014.

\bibitem{sedghi2014provable}
H.~Sedghi and A.~Anandkumar.
\newblock Provable methods for training neural networks with sparse
  connectivity.
\newblock {\em arXiv preprint arXiv:1412.2693}, 2014.

\bibitem{janzamin2015beating}
M.~Janzamin, H.~Sedghi, and A.~Anandkumar.
\newblock Beating the perils of non-convexity: Guaranteed training of neural
  networks using tensor methods.
\newblock {\em arXiv preprint arXiv:1506.08473}, 2015.

\bibitem{haeffele2015global}
B.~D Haeffele and R.~Vidal.
\newblock Global optimality in tensor factorization, deep learning, and beyond.
\newblock {\em arXiv preprint arXiv:1506.07540}, 2015.

\bibitem{gautier2016globally}
A.~Gautier, Q.~N. Nguyen, and M.~Hein.
\newblock Globally optimal training of generalized polynomial neural networks
  with nonlinear spectral methods.
\newblock In {\em NIPS}, pages 1687--1695, 2016.

\bibitem{brutzkus2017globally}
A.~Brutzkus and A.~Globerson.
\newblock Globally optimal gradient descent for a convnet with gaussian inputs.
\newblock {\em arXiv preprint arXiv:1702.07966}, 2017.

\bibitem{soltanolkotabi2017learning}
M.~Soltanolkotabi.
\newblock Learning relus via gradient descent.
\newblock In {\em NIPS}, pages 2004--2014, 2017.

\bibitem{soudry2017exponentially}
D.~Soudry and E.~Hoffer.
\newblock Exponentially vanishing sub-optimal local minima in multilayer neural
  networks.
\newblock {\em arXiv preprint arXiv:1702.05777}, 2017.

\bibitem{goel2017learning}
S.~Goel and A.~Klivans.
\newblock Learning depth-three neural networks in polynomial time.
\newblock {\em arXiv preprint arXiv:1709.06010}, 2017.

\bibitem{du2017convolutional}
S.~S. Du, J.~D. Lee, and Y.~Tian.
\newblock When is a convolutional filter easy to learn?
\newblock {\em arXiv preprint arXiv:1709.06129}, 2017.

\bibitem{zhong2017recovery}
K.~Zhong, Z.~Song, P.~Jain, P.~L Bartlett, and I.~S Dhillon.
\newblock Recovery guarantees for one-hidden-layer neural networks.
\newblock {\em arXiv preprint arXiv:1706.03175}, 2017.

\bibitem{li2017convergence}
Y.~Li and Y.~Yuan.
\newblock Convergence analysis of two-layer neural networks with relu
  activation.
\newblock In {\em NIPS}, pages 597--607, 2017.

\bibitem{liang2018understanding}
S.~Liang, R.~Sun, Y.~Li, and R.~Srikant.
\newblock Understanding the loss surface of neural networks for binary
  classification.
\newblock 2018.

\bibitem{haeffele2014structured}
B.~Haeffele, E.~Young, and R.~Vidal.
\newblock Structured low-rank matrix factorization: Optimality, algorithm, and
  applications to image processing.
\newblock In {\em ICML}, 2014.

\bibitem{soudry2016no}
D.~Soudry and Y.~Carmon.
\newblock No bad local minima: Data independent training error guarantees for
  multilayer neural networks.
\newblock {\em arXiv preprint arXiv:1605.08361}, 2016.

\bibitem{shamir2018resnets}
O.~Shamir.
\newblock Are resnets provably better than linear predictors?
\newblock {\em arXiv preprint arXiv:1804.06739}, 2018.

\bibitem{zhang2016understanding}
C.~Zhang, S.~Bengio, M.~Hardt, B.~Recht, and O.~Vinyals.
\newblock Understanding deep learning requires rethinking generalization.
\newblock {\em ICLR}, 2016.

\bibitem{safran2017spurious}
Itay Safran and Ohad Shamir.
\newblock Spurious local minima are common in two-layer relu neural networks.
\newblock {\em ICML}, 2018.

\bibitem{zhang2012best}
X.~Zhang, C.~Ling, and L.~Qi.
\newblock The best rank-1 approximation of a symmetric tensor and related
  spherical optimization problems.
\newblock {\em SIAM Journal on Matrix Analysis and Applications}, 2012.

\end{thebibliography}
\bibliographystyle{unsrt}

\begin{appendix}

\newpage
\section{Proof of Lemma~\ref{lemma::a=0}}

\subsection{Proof of Lemma~\ref{lemma::a=0}($i$)}
\begin{proof}
To prove $a^{*}=0$, we only need to check the first order conditions of local minima. By assumption that $\tilde{\bm{\theta}}^{*}=(\bm{\theta}^{*}, a^{*},\bm{w}^{*},b^{*})$ is a local minimum of $\tilde{L}_{n}$, then the derivative of $\tilde{L}_{n}$ with respect to $a$ and $b$ at the point $\tilde{\bm{\theta}}^{*}$ are all zeros, i.e., 
	\begin{align}
	\left.\nabla_ a \tilde{L}_{n}(\tilde{\bm{\theta}})\right|_{\tilde{\bm{\theta}}=\tilde{\bm{\theta}}^{*}}  &=-\sum_{i=1}^n \ell'\left( -y_if(x_i;\bm{\theta}^{*})- y_ia e^{{\bm{w}^{*}}^\top x_i +b^{*}}\right) y_i  \exp({\bm{w}^{*}}^\top x_i +b^{*}) +\lambda a^{*}=0, \notag\\
	\left.\nabla_b \tilde{L}_{n}(\tilde{\bm{ \theta}})\right|_{\tilde{\bm{\theta}}=\tilde{\bm{\theta}}^{*}} &=- a^{*} \sum_{i=1}^n \ell'\left( -y_if(x_i;\bm{\theta}^{*})- y_ia e^{{\bm{w}^{*}}^\top x_i +b^{*}}\right) y_i  \exp({\bm{w}^{*}}^\top x_i+b^{*})=0. \notag
	\end{align}
From above two equations, it is not difficult to see that $a^{*}$ satisfies $\lambda {a^{*}}^{2}=0$ or, equivalently, $a^{*}=0$.  

\end{proof}

\subsection{Proof of Lemma~\ref{lemma::a=0}($ii$)}
\begin{proof}
The main idea of the proof is to use the high order information of the local minimum to prove the Lemma. Due to the assumption that $\tilde{\bm{\theta}}=(\bm{\theta}^{*}, a^{*}, \bm{w}^{*}, b^{*})$ is a local minimum of the empirical loss function $\tilde{L}_{n}$, there exists a bounded local region such that the parameters $\tilde{\bm{\theta}}^{*}$ achieve the minimum loss value in this region, i.e.,  $\exists\delta\in(0,1)$ such that $\tilde{L}_{n}(\tilde{\bm{\theta}}^{*}+\bm{\Delta})\ge\tilde{L}_{n}(\tilde{\bm{\theta}}^{*})$ for $\forall\bm{\Delta}: \|\bm{\Delta}\|_{2}\le\delta$. 

Now, we use $\delta_{a}$, $\bm{\delta_{w}}$ to denote the perturbations on the parameters $a$ and $\bm{w}$, respectively. Next, we consider the loss value at the point $\tilde{\bm{\theta}}^{*}+\bm{\Delta}=(\bm{\theta}^{*}, a^{*}+\delta_{a}, \bm{w}^{*}+\bm{\delta_{w}}, b^{*})$, where we set $|\delta_{a}|=e^{-1/\e}$ and $\bm{\delta_{w}}=\e\bm{u}$ for an arbitrary unit vector $\bm{u}:\|\bm{u}\|_{2}=1$. Therefore, as $\e$ goes to zero, the perturbation magnitude $\|\bm{\Delta}\|_{2}$ also goes to zero and this indicates that there exists an $\e_{0}\in(0,1)$ such that  $\tilde{L}_{n}(\tilde{\bm{\theta}}^{*}+\bm{\Delta})\ge\tilde{L}_{n}(\tilde{\bm{\theta}}^{*})$ for $\forall \e\in[0, \e_{0})$. By $a^{*}=0$, the output of the model $\tilde{f}$ under parameters $\tilde{\bm{\theta}}^{*}+\bm{\Delta}$ can be expressed by 
$$\tilde{f}(x;\tilde{\bm{\theta}}^{*}+\bm{\Delta})=f(x;\bm{\theta}^{*})+\delta_{a}\exp(\bm{\delta_{w}}^{\top}x)\exp({\bm{w}^{*}}^{\top}x+b^{*}).$$
        Let $g(x_{i};\bm{w}^{*}, \bm{\delta_{w}}^{}, b^{*})=\exp({\bm{\delta_{w}}^{}}^{\top}x_{i})\exp({\bm{w}^{*}}^{\top}x_{i}+b^{*})$.
        For each sample $(x_{i}, y_{i})$ in the dataset, by the second order Taylor expansion with Lagrangian remainder, there exists a scalar $\xi_{i}\in[-|\delta_{a}|, |\delta_{a}|]$ depending on $\delta_{a}$ and $g(x_{i};\bm{w}^{*}, \bm{\delta_{w}}, b^{*})$ such that the following equation holds,
        \begin{align*}
        \ell(-y_{i}f(x_{i};\bm{\theta}^{*})-y_{i}&\delta_{a}g(x_{i};\bm{w}^{*}, \bm{\delta_{w}}, b^{*}))\\
        &=\ell(-y_{i}f(x_{i};\bm{\theta}^{*}))+\ell'(-y_{i}f(x_{i};\bm{\theta}^{*}))(-y_{i})\delta_{a}g(x_{i};\bm{w}^{*}, \bm{\delta_{w}}, b^{*})\\
        &\quad+\frac{1}{2!}\ell''(-y_{i}f(x_{i};\bm{\theta}^{*})-y_{i}\xi_{i}g(x_{i};\bm{w}^{*},\bm{\delta_{w}}, b^{*}))\delta^{2}_{a}g^{2}(x_{i};\bm{w}^{*}, \bm{\delta_{w}}, b^{*}).
        \end{align*}
        Let vector $\bm{u}\in\mathbb{R}$ denote an arbitrary unit vector. Let $|\delta_{a}|=\exp(-1/\varepsilon)$ and $\bm{\delta_{w}}=\varepsilon \bm{u}$. Clearly, for all $\varepsilon<1$,  $|\delta_{a}|<e^{-1}$ and $\|\bm{\delta_{{w}}}\|_{2}<1$, we have 
        $$g(x_{i};\bm{w}^{*}, \bm{\delta_{w}}, b^{*})=\exp(\bm{\delta_{w}}^{\top}x_{i})\exp({\bm{w}^{*}}^{\top}x_{i}+b^{*})\le \exp(\|x_{i}\|_{2})\exp({\bm{w}^{*}}^{\top}x_{i}+b^{*}).$$
        Since $|\xi_{i}|<|\delta_{a}|<e^{-1}$, then for each $i\in[n]$, there exists a constant $C_{i}$ depend on $\bm{\theta}^{*}, \bm{w}^{*}, b^{*}$ such that 
        $$|\ell''(-y_{i}f(x_{i};\bm{\theta}^{*})-z)|< C_{i}$$
        holds for all $z\in[-\exp(-1+\|x_{i}\|_{2}+{\bm{w}^{*}}^{\top}x_{i}+b^{*}), \exp(-1+\|x_{i}\|_{2}+{\bm{w}^{*}}^{\top}x_{i}+b^{*})]$.
        
        Since $\tilde{\bm{\theta}}^{*}$ is a local minimum, then there exists $\varepsilon_{0}\in(0, 1)$ such that the inequality
        \begin{align*}
        L_{n}(\tilde{\bm{\theta}}^{*}+\bm{\Delta}) - L_{n}(\tilde{\bm{\theta}}^{*})&=
        \sum_{i=1}^{n}\ell(-y_{i}f(x_{i};\bm{\theta}^{*})-y_{i}\delta_{a}g(x_{i};\bm{w}^{*}, \bm{\delta_{w}}, b^{*}))+\frac{\lambda\delta_{a}^{2}}{2}-\sum_{i=1}^{n}\ell(-y_{i}f(x_{i};\bm{\theta}^{*}))\\
        &=\sum_{i=1}^{n}\ell'(-y_{i}f(x_{i};\bm{\theta}^{*}))(-y_{i})\delta_{a}g(x_{i};\bm{w}^{*}, \bm{\delta_{w}}, b^{*})+\frac{\lambda\delta_{a}^{2}}{2}\\
        &\quad+\sum_{i=1}^{n}\frac{1}{2!}\ell''(-y_{i}f(x_{i};\bm{\theta}^{*})-y_{i}\xi_{i}g(x_{i};\bm{w}^{*},\bm{\delta_{w}}, b^{*}))\delta^{2}_{a}g^{2}(x_{i};\bm{w}^{*}, \bm{\delta_{w}}, b^{*})\\
        &\ge 0
        \end{align*}
        holds for all $\varepsilon<\varepsilon_{0}$. In addition, we have 
        \begin{align*}
        \sum_{i=1}^{n}\ell''(-y_{i}f(x_{i};\bm{\theta}^{*})-&y_{i}\xi_{i}g(x_{i};\bm{w}^{*},\bm{\delta_{w}}, b^{*}))\delta^{2}_{a}g^{2}(x_{i};\bm{w}^{*}, \bm{\delta_{w}}, b^{*})\\
        &\le \sum_{i=1}^{n}C_{i}\delta^{2}_{a}g^{2}(x_{i};\bm{w}^{*}, \bm{\delta_{w}}, b^{*})\\
        &\le \exp(-2/\varepsilon)\sum_{i=1}^{n}C_{i}\exp(-\|x_{i}\|_{2}+{\bm{w}^{*}}^{\top}x_{i}+b^{*})
        \end{align*}
       Recall that scalar $C_{i}$ only depends on $\tilde{\bm{\theta}}^{*}=(\bm{\theta}^{*}, \bm{w}^{*}, b^{*})$ and $x_{i}$, thus the scalar $C(\tilde{\bm{\theta}}^{*},\mathcal{D})=\sum_{i=1}^{n}C_{i}\exp(-\|x_{i}\|_{2}+{\bm{w}^{*}}^{\top}x_{i}+b^{*})+\lambda/2$ can be viewed as a scalar depending only on parameters $\tilde{\bm{\theta}}^{*}$ and dataset $\mathcal{D}$. Thus, for any $\varepsilon:\varepsilon<\varepsilon_{0}$ and for any $\sgn(\delta_{a})\in\{-1, 1\}$, the inequality 
       \begin{equation}
       \sgn(\delta_{a})\exp(-1/\varepsilon)\sum_{i=1}^{n}\ell'(-y_{i}f(x_{i};\bm{\theta}^{*}))(-y_{i})g(x_{i};\bm{w}^{*}, \varepsilon\bm{u}, b^{*})+C(\tilde{\bm{\theta}}^{*},\mathcal{D})\exp(-2/\varepsilon)\ge 0
       \end{equation}
       always holds.  
       This indicates that for any $\varepsilon:\varepsilon<\varepsilon_{0}$ and for any $\sgn(\delta_{a})\in\{-1, 1\}$, the inequality 
        \begin{equation}\label{lemma::eq-2}
       \sgn(\delta_{a})\sum_{i=1}^{n}\ell'(-y_{i}f(x_{i};\bm{\theta}^{*}))(-y_{i})\exp(\varepsilon\bm{u}^{\top}x_{i})\exp({\bm{w}^{*}}^{\top}x_{i}+b^{*})+C(\tilde{\bm{\theta}}^{*},\mathcal{D})\exp(-1/\varepsilon)\ge 0
       \end{equation}
       always holds. 
       We now proceed by induction. For the base case where $p=0$, 
       for each $\sgn(\delta_{a})\in\{-1, 1\}$, we take the limit on the both sides of inequality~\eqref{lemma::eq-2} as $\varepsilon\rightarrow 0$ and thus obtain   
       \begin{equation}
       \sgn(\delta_{a})\sum_{i=1}^{n}\ell'(-y_{i}f(x_{i};\bm{\theta}^{*}))(-y_{i})\exp({\bm{w}^{*}}^{\top}x_{i}+b^{*})\ge 0,
       \end{equation}
       which further establishes the base case
       \begin{equation}
       \sum_{i=1}^{n}\ell'(-y_{i}f(x_{i};\bm{\theta}^{*}))(-y_{i})\exp({\bm{w}^{*}}^{\top}x_{i}+b^{*}) = 0.
 	\end{equation}
	The inductive hypothesis is that the equality 
	\begin{equation}\label{lemma::eq-3}
	\sum_{i=1}^{n}\ell'(-y_{i}f(x_{i};\bm{\theta}^{*}))(-y_{i})\exp({\bm{w}^{*}}^{\top}x_{i}+b^{*})(\bm{u}^{\top}x_{i})^{j} = 0
	\end{equation}
	 holds for all $j=0,...,k-1$. Now we need to prove that the equality~\eqref{lemma::eq-3} holds for $j=k$. 
	 Since the equality holds for all $j= 0,..., k-1$, then we have 
	 \begin{align}
	 \sgn(\delta_{a})\sum_{i=1}^{n}\ell'(-y_{i}f(x_{i};\bm{\theta}^{*}))(-y_{i})\exp({\bm{w}^{*}}^{\top}x_{i}+b^{*})&\left[\frac{\exp(\varepsilon\bm{u}^{\top}x_{i})-\sum_{j=0}^{k-1}\frac{(\varepsilon \bm{u}^{\top}x_{i})^{j}}{j!}}{\varepsilon^{k}}\right]\notag\\
	 &+C(\tilde{\bm{\theta}}^{*},\mathcal{D})1/\varepsilon^{k}\exp(-1/\varepsilon)\ge 0\label{lemma::eq-4}
	 \end{align}
	 Taking the limit on the both sides of Eq.~\eqref{lemma::eq-4}, we obtain that the inequality
	 \begin{equation}\label{lemma::eq-5}
	\sgn(\delta_{a})\sum_{i=1}^{n}\ell'(-y_{i}f(x_{i};\bm{\theta}^{*}))(-y_{i})\exp({\bm{w}^{*}}^{\top}x_{i}+b^{*})(\bm{u}^{\top}x_{i})^{k} \ge 0,
	\end{equation}
	holds for every $\sgn(\delta_{a})\in\{-1, 1\}$ and this further implies 
	\begin{equation}\label{lemma::eq-6}
	\sum_{i=1}^{n}\ell'(-y_{i}f(x_{i};\bm{\theta}^{*}))(-y_{i})\exp({\bm{w}^{*}}^{\top}x_{i}+b^{*})(\bm{u}^{\top}x_{i})^{k} = 0.
	\end{equation}
	Thus Eq.~\eqref{lemma::eq-6} finishes our induction. 
\end{proof}
	
\section{Proof of Lemma~\ref{lemma::tensor}}
\begin{proof}
It is easy to check that the tensor $T_{k}=\sum_{i=1}^{n}c_{i}x_{i}^{\otimes k}$ is a symmetric tensor. From Theorem~1 in reference~\cite{zhang2012best}, we directly have 
$$\max_{\bm{u_{1},...,u_{k}}:\|\bm{u_{1}}\|_{2}=...=\|\bm{u_{k}}\|_{2}=1}|T_{k}(\bm{u}_{1},...,\bm{u}_{k})|=\max_{\bm{u}:\|\bm{u}\|_{2}=1}|T_{k}(\bm{u},...,\bm{u})|.$$
Since 
$$T_{k}(\bm{u},...,\bm{u})=\sum_{i=1}^{n}c_{i}(\bm{u}^{\top}x_{i})^{k}=0$$ 
holds for all $\|\bm{u}\|_{2}=1$, then 
$$\max_{\bm{u_{1},...,u_{k}}:\|\bm{u_{1}}\|_{2}=...=\|\bm{u_{k}}\|_{2}=1}|T_{k}(\bm{u}_{1},...,\bm{u}_{k})|=0$$
which is equivalent to 
$$T_{k}=\bm{0}^{\otimes k}.$$
\end{proof}

\textbf{Remark: }One implication of Lemma~\ref{lemma::a=0} and \ref{lemma::tensor} is that for any non-negative integer $p\ge 0$, the following $p$-th order tensor is a zero tensor, 
$$\sum_{i=1}^{n}\ell'(-y_{i}f(x_{i};\bm{\theta}^{*}))(-y_{i})\exp({\bm{w}^{*}}^{\top}x_{i}+b^{*})x_{i}^{\otimes p}= \bm{0}_{d}^{\otimes}.$$
This further indicates for any monomial $\pi:\mathbb{R}^{d}\rightarrow \mathbb{R}$, 
$$\sum_{i=1}^{n}\ell'(-y_{i}f(x_{i};\bm{\theta}^{*}))(-y_{i})\exp({\bm{w}^{*}}^{\top}x_{i}+b^{*})\pi(x_{i})= 0.$$

\section{Proof  of Theorem~\ref{thm::single-exp}}
\begin{proof}
For every dataset $\mathcal{D}$ satisfying Assumption~\ref{assump::realizability}, by the Lagrangian interpolating polynomial, there always exists a polynomial  $P(x)=\sum_{j}c_{j}\pi_{j}(x)$ defined on $\mathbb{R}^{d}$ such that it can correctly classify all samples in the dataset with margin at least one, i.e., $y_{i}P(x_{i})\ge 1,\forall i\in[n]$, where $\pi_{j}$ denotes the $j$-th monomial in the polynomial $P(x)$. 
Therefore, from Lemma~\ref{lemma::a=0} and \ref{lemma::tensor}, it follows that 
\begin{align*}
\sum_{i=1}^{n}\ell'(-y_{i}f(x_{i};\bm{\theta}^{*}))e^{{\bm{w}^{*}}^{\top}x_{i}+b^{*}}y_{i}P(x_{i})=\sum_{j}c_{j}\sum_{i=1}^{n}\ell'(-y_{i}f(x_{i};\bm{\theta}^{*}))y_{i}e^{{\bm{w}^{*}}^{\top}x_{i}+b^{*}}\pi_{j}(x_{i})=0.
\end{align*}
Since $y_{i}P(x_{i})\ge 1$ and $e^{{\bm{w}^{*}}^{\top}x_{i}+b^{*}}>0$ hold for  $\forall i\in[n]$ and the loss function $\ell$ is a non-decreasing function, i.e., $\ell'(z)\ge 0,\forall z\in\mathbb{R}$, then $\ell'(-y_{i}f(x_{i};\bm{\theta}^{*}))=0$ holds for all $i\in[n]$. In addition, from the assumption that every critical point of the loss function $\ell$ is a global minimum, it follows that $z_{i}=-y_{i}f(x_{i};\bm{\theta}^{*})$ achieves the global minimum of the loss function $\ell$ and this further indicates that $\bm{\theta}^{*}$ is a global minimum of the empirical loss $L_{n}(\bm{\theta})$. Furthermore, since at every local minimum, the exponential neuron is inactive, $a^{*}=0$, then the set of parameters $\tilde{\bm{\theta}}^{*}$ is a global minimum of the loss function $\tilde{L}(\tilde{\bm{\theta}})$. Finally, since every critical point of the loss function $\ell(z)$ satisfies $z<0$, then for every sample, $\ell'(-y_{i}f(x_{i};\bm{\theta}^{*}))=0$ indicates that $y_{i}f(x_{i};\bm{\theta}^{*})>0$, or, equivalently, $y_{i}=\sgn(f(x_{i};\bm{\theta}^{*}))$. Therefore, the set of parameters $\bm{\theta}^{*}$ also minimizes the training error. In summary, the set of parameters $\tilde{\bm{\theta}}^{*}=(\bm{\theta}^{*},a^{*}, \bm{w}^{*}, b^{*})$ minimizes the  loss function $\tilde{L}_{n}(\tilde{\bm{\theta}})$ and the set of parameters $\bm{\theta}^{*}$ simultaneously minimizes the empirical loss function $L_{n}(\bm{\theta})$ and the training error $R_{n}(\bm{\theta};f)$. 

\end{proof}

\section{Proof of Corollary~\ref{cor::single-exp}}
\begin{proof}
The proof follows directly from the proof of Lemma~\ref{lemma::a=0}($i$). From Lemma~\ref{lemma::a=0}($i$), it follows that at every local minimum $\tilde{\bm{\theta}}^{*}=(\bm{\theta}^{*},\bm{w}^{*}, b^{*})$, the exponential neuron is inactive $a^{*}=0$. This indicates that at this local minimum, 
$$\tilde{f}(x;\tilde{\bm{\theta}}^{*})=f(x;\bm{\theta}^{*})+a^{*}\exp({\bm{w}^{*}}x+b^{*})=f(x;\bm{\theta}^{*}),\quad\forall x\in\mathbb{R}^{d}.$$
Therefore, two networks $\tilde{f}(\cdot;\tilde{\bm{\theta}}^{*})$ and $f(\cdot;\bm{\theta})$ are equivalent.  

\end{proof}

\newpage
\section{Proof of Theorem~\ref{thm::multi-exp}}

\subsection{Notations and Important Lemmas}

\textbf{Notations.} Let $M_{l}$ denote the number of neurons in the $l$-th layer of the original neural network and thus $M_{l}+1$ is the number of neurons in the $l$-th layer of the augmented neural network where we add an additional exponential neuron to each layer. Let $a=(a_{1},...,a_{M_{L}+1})$ denote the weight of the output layer where $a_{M_{L}+1}$ is the weight of the exponential neuron in the last layer. Let $w_{j,k}^{(l)}$ denote the weight connecting the $j$-th neuron in the $l$-th layer and the $k$-th neuron in the $(l-1)$-th layer.
\begin{figure}[t]
\centering
\includegraphics[width = 0.6\linewidth]{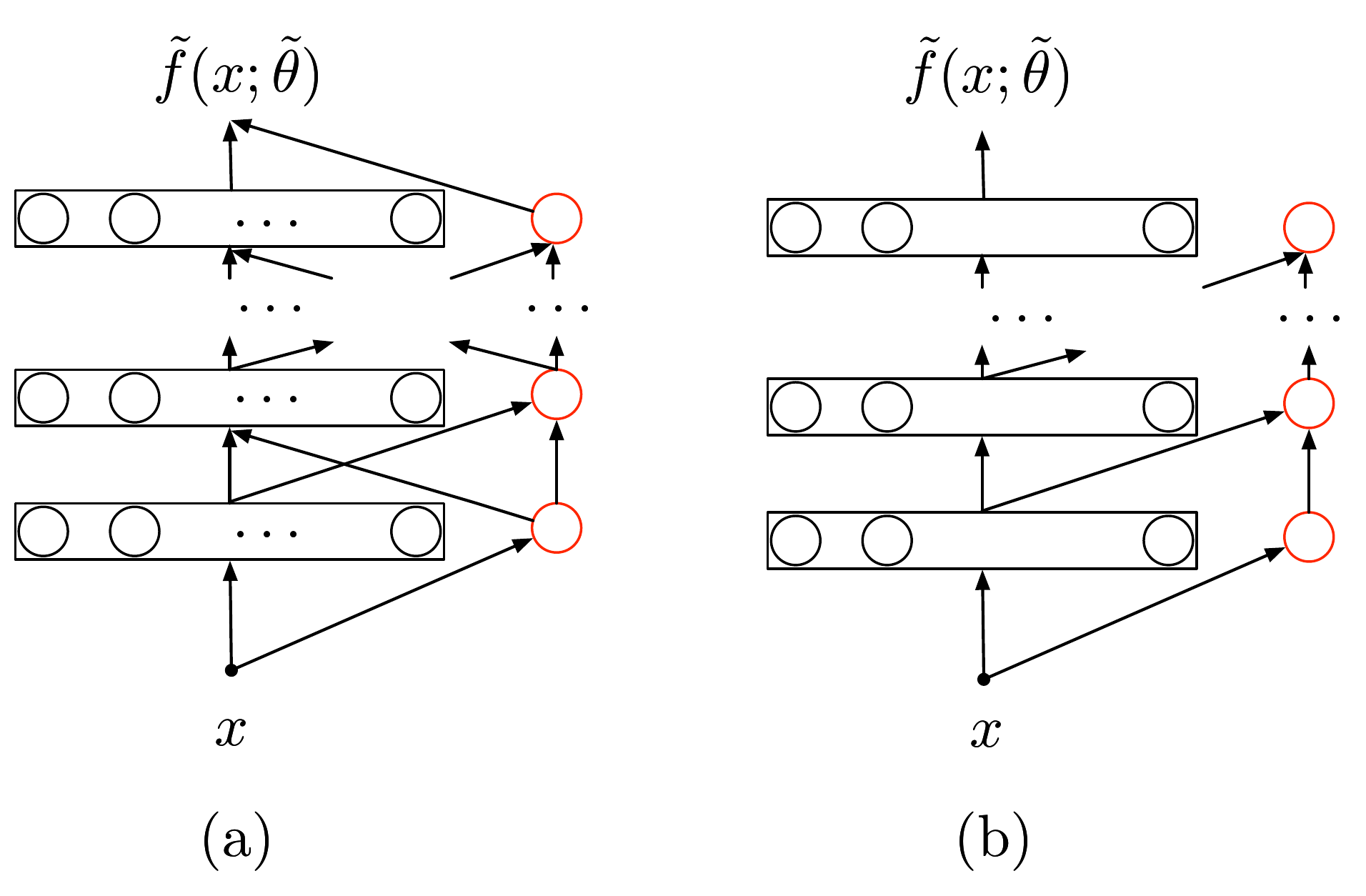}

\caption{(a) The network architecture. (b) The network architecture at any local minimum, where all exponential neurons do not contribute to the final output of the neural network. Black circles represent the neuron activation in the original network $f(x;\bm{\theta})$ and red circles represent the exponential neuron.}
\label{fig::structure2}
\end{figure}

\setcounter{lemma}{2}
\begin{lemma}\label{lemma::multilayer-zero}
If $\tilde{\bm{\theta}}^{*}$ is a local minimum of the empirical loss function $\tilde{L}_{n}(\tilde{\bm{\theta}})$, then 
\begin{itemize}
\item[(1)]  $a^{*}_{M_{L}+1}=0$
\item[(2)]  ${w^{(l)}_{j, M_{l-1}+1}}^{*}=0$ for $l=2,..., L$ and $j= 1,...,M_{l}+1$.
\end{itemize}
\end{lemma}
\textbf{Remark:} Lemma~\ref{lemma::multilayer-zero} shows that at every local minimum, the feedforward neural network shown in Fig.~\ref{fig::structure2}(a) becomes the neural network shown in Fig.~\ref{fig::structure2}(b). This means that all exponential neurons are inactive and thus do no contribute to the final output of the neural network. 

\begin{lemma} \label{lemma::empirical-loss}
If $\tilde{\bm{\theta}}^{*}$ is a local minimum of the empirical loss function $\tilde{L}_{n}(\tilde{\bm{\theta}})$, then there exists a function $\eta:\mathbb{R}^{d}\times\mathbb{R}^{|\tilde{\bm{\theta}}|}\rightarrow\mathbb{R}$ satisfying $\eta(x,\tilde{\bm{\theta}})>0$ for all $(x,\tilde{\bm{\theta}})$ such that for any integer $p\ge0$, the equation
\begin{equation}
\sum_{i=1}^{n}\ell'\left(-y_{i}\tilde{f}(x_{i};\tilde{\bm{\theta}}^{*})\right)(-y_{i})\eta(x_{i},\tilde{\bm{\theta}}^{*})(\bm{u}^{\top}x_{i}+v)^{p}=0
\end{equation}
holds for any vector $\bm{u}$ and  scalar $v$ satisfying $\|\bm{u}\|_{2}^{2}+v^{2}=1$.
\end{lemma}

\subsection{Proof of Lemma~\ref{lemma::multilayer-zero}}
\begin{proof}
\textbf{(1)} We first prove that the parameter $a_{M_{L}+1}^{*}$ in the last layer is zero, i.e., $a_{M_{L}+1}^{*}=0$. Since $\tilde{\bm{\theta}}^{*}$ is a local minimum, then the derivatives of the empirical loss with respect to the parameter $a_{M_{L}+1}$ and $b_{M_{L}+1}$ are both zeros, 
\begin{align}
\nabla_{a_{M_{L}+1}}L(\tilde{\bm{\theta}}^{*}) &= \sum_{i=1}^{n}\ell'(-y_{i}\tilde{f}(x_{i};\tilde{\bm{\theta}}^{*}))(-y_{i})\exp\left({\bm{w}_{L}^{*}}^{\top}\bm{z}^{(L-1)}(x_{i};\tilde{\bm{\theta}}^{*})+b^{*}_{M_{L}+1}\right)\\
&\quad+L\lambda \left(a^{*}_{M_{L}+1}\right)^{2L-1}=0\label{eq::lemma3-1},\\
\nabla_{b_{M_{L}+1}}L(\tilde{\bm{\theta}}^{*}) &= a_{M_{L}+1}\sum_{i=1}^{n}\ell'(-y_{i}\tilde{f}(x_{i};\tilde{\bm{\theta}}^{*}))(-y_{i})\exp\left({\bm{w}_{L}^{*}}^{\top}\bm{z}^{(L-1)}(x_{i};\tilde{\bm{\theta}}^{*})+b^{*}_{M_{L}+1}\right)=0\label{eq::lemma3-2},
\end{align}
where $\bm{w}_{L}, b_{M_{L+1}}$ denotes the weight vector and bias scalar of the exponential neuron in the $L$-th layer, respectively and each component of the vector $\bm{z}^{(L-1)}(x_{i};\bm{\theta}^{*})$ is an output from the $(L-1)$-th layer on the sample $x_{i}$. 
Combining Eq.~\eqref{eq::lemma3-1} and~\eqref{eq::lemma3-2}, we have 
$$a^{*}_{M_{L}+1}=0.$$

\textbf{(2)} We next prove that the parameters ${w^{(l)}_{j, M_{l-1}+1}}^{*}=0$ for all $l=2,..., L$ and $j= 1,...,M_{l}+1$, where  $w^{(l)}_{j, M_{l-1}+1}$ denotes the weight on the connection between the $j$-th neuron in the $l$-th layer and $(M_{l-1}+1)$-th neuron in the $(l-1)$-th layer. For  neurons in the $l$-th layer, we have 
\begin{align}
z_{j}^{(l)}(x)&=\sigma\left({\bm{w}_{j}^{(l)}}^{\top}\bm{z}^{(l-1)}(x)+b_{j}^{(l)}\right),\quad j = 1,..., M_{l}\\
z_{j}^{(l)}(x)&=\exp\left({\bm{w}_{j}^{(l)}}^{\top}\bm{z}^{(l-1)}(x)+b_{j}^{(l)}\right), \quad j = M_{l}+1,
\end{align}
where $\bm{w}_{j}^{(l)}$ and $b_{j}^{(l)}$ denotes the weight vector and the bias scalar of the $j$-th neuron in the $l$-th layer. 
Therefore, taking the derivative of the empirical loss with respect to each $w^{(l)}_{j, M_{l-1}+1}$, $j=1,...,M_{l}+1$ in the $l$-th layer, we have
\begin{align}
0&=\nabla_{w^{(l)}_{j, M_{l-1}+1}}L(\tilde{\bm{\theta}}^{*})\\
 &= \sum_{i=1}^{n}\ell'(-y_{i}\tilde{f}(x_{i};\tilde{\bm{\theta}}^{*}))(-y_{i})\frac{\partial \tilde{f}(x_{i};\tilde{\bm{\theta}}^{*})}{\partial w^{(l)}_{j, M_{l-1}+1}}+L\lambda\left( {w^{(l)}_{j, M_{l-1}+1}}^{*}\right)^{2L-1}\notag\\
&= \sum_{i=1}^{n}\ell'(-y_{i}\tilde{f}(x_{i};\tilde{\bm{\theta}}^{*}))(-y_{i})\frac{\partial \tilde{f}(x_{i};\tilde{\bm{\theta}}^{*})}{\partial z^{(l)}_{j}}\frac{\partial z^{(l)}_{j}(x_{i})}{\partial w^{(l)}_{j, M_{l-1}+1}}+L\lambda\left( {w^{(l)}_{j, M_{l-1}+1}}^{*}\right)^{2L-1}\notag\\
&=\sum_{i=1}^{n}\ell'(-y_{i}\tilde{f}(x_{i};\tilde{\bm{\theta}}^{*}))(-y_{i})\frac{\partial \tilde{f}(x_{i};\tilde{\bm{\theta}}^{*})}{\partial z^{(l)}_{j}}\frac{\partial z^{(l)}_{j}}{\partial b_{j}^{(l)}}z_{M_{l-1}+1}^{(l-1)}(x_{i})+L\lambda\left( {w^{(l)}_{j, M_{l-1}+1}}^{*}\right)^{2L-1},\label{eq::lemma3-3}
\end{align}
where in the last equality, we used the property that the equality
\begin{equation}
\frac{\partial z_{j}^{(l)}(x_{i})}{\partial w^{(l)}_{j, M_{l-1}+1}}=\sigma'\left({{\bm{w}_{j}^{(l)}}^{*}}^{\top}\bm{z}^{(l-1)}(x)+{b_{j}^{(l)}}^{*}\right)z_{M_{l-1}+1}^{(l-1)}(x_{i})=\frac{\partial z_{j}^{(l)}(x_{i})}{\partial b_{j}^{(l)}}z_{M_{l-1}+1}^{(l-1)}(x_{i})
\end{equation}
holds for $j=1,...,M_{l}$ and 
\begin{equation}
\frac{\partial z_{M_{l}+1}^{(l)}(x_{i})}{\partial w^{(l)}_{M_{l}+1, M_{l-1}+1}}=\exp\left({{\bm{w}_{M_{l}+1}^{(l)}}^{*}}^{\top}\bm{z}^{(l-1)}(x)+{b_{M_{l}+1}^{(l)}}^{*}\right)z_{M_{l-1}+1}^{(l-1)}(x_{i})=\frac{\partial z_{M_{l}+1}^{(l)}(x_{i})}{\partial b_{M_{l}+1}^{(l)}}z_{M_{l-1}+1}^{(l-1)}(x_{i})
\end{equation}

Furthermore, taking the derivative with respect to $b^{(l-1)}_{M_{l-1}+1}$, we have 

\begin{align}
0&=\nabla_{b^{(l-1)}_{M_{l-1}+1}}L(\tilde{\bm{\theta}}^{*})\notag\\
&= \sum_{i=1}^{n}\ell'(-y_{i}\tilde{f}(x_{i};\tilde{\bm{\theta}}^{*}))(-y_{i})\frac{\partial \tilde{f}(x_{i};\tilde{\bm{\theta}}^{*})}{\partial b^{(l-1)}_{M_{l-1}+1}}\notag\\
&= \sum_{i=1}^{n}\ell'(-y_{i}\tilde{f}(x_{i};\tilde{\bm{\theta}}^{*}))(-y_{i})\left[\sum_{j=1}^{M_{l}+1}\frac{\partial \tilde{f}(x_{i};\tilde{\bm{\theta}}^{*})}{\partial z^{(l)}_{j}}\frac{\partial z^{(l)}_{j}(x_{i})}{\partial b^{(l-1)}_{M_{l-1}+1}}\right]\notag\\
&= \sum_{i=1}^{n}\ell'(-y_{i}\tilde{f}(x_{i};\tilde{\bm{\theta}}^{*}))(-y_{i})\left[\sum_{j=1}^{M_{l}+1}\frac{\partial \tilde{f}(x_{i};\tilde{\bm{\theta}}^{*})}{\partial z^{(l)}_{j}}\cdot\frac{\partial z^{(l)}_{j}(x_{i})}{\partial z_{M_{l-1}+1}^{(l-1)}}\cdot\frac{\partial z_{M_{l-1}+1}^{(l-1)}(x_{i})}{\partial b^{(l-1)}_{M_{l-1}+1}}\right]\notag\\
&=\sum_{i=1}^{n}\ell'(-y_{i}\tilde{f}(x_{i};\tilde{\bm{\theta}}^{*}))(-y_{i})\left[\sum_{j=1}^{M_{l}+1}\frac{\partial \tilde{f}(x_{i};\tilde{\bm{\theta}}^{*})}{\partial z^{(l)}_{j}}\cdot\frac{\partial z_{j}^{(l)}(x_{i})}{\partial b_{j}^{(l)}}w_{j,M_{l-1}+1}^{(l)}\cdot z_{M_{l-1}+1}^{(l-1)}(x_{i})\right],\label{eq::lemma3-6}
\end{align}
where in the last equality, we used the following equality 
\begin{align}
\frac{\partial z^{(l)}_{j}(x_{i})}{\partial z_{M_{l-1}+1}^{(l-1)}}=\frac{\partial z_{j}^{(l)}(x_{i})}{\partial b_{j}^{(l)}}w_{j,M_{l-1}+1}^{(l)} 
\end{align}
and the property of the exponential neuron,
\begin{align}
\frac{\partial z_{M_{l-1}+1}^{(l-1)}(x_{i})}{\partial b^{(l-1)}_{M_{l-1}+1}}=\exp\left({{\bm{w}_{M_{l-1}+1}^{(l)}}^{*}}^{\top}\bm{z}^{(l-2)}(x_{i})+{b_{M_{l-1}+1}^{(l-1)}}^{*}\right)= z_{M_{l-1}+1}^{(l-1)}
\end{align}

Now we  multiply the weight ${w_{j, M_{l-1}+1}^{(l)}}^{*}$ on the both sides of Eq.~\eqref{eq::lemma3-3}, respectively and  sum them together over $j=1,...,M_{l}$. 
Thus, we obtain 
\begin{align}\label{eq::lemma3-5}
\sum_{i=1}^{n}\ell'(-y_{i}\tilde{f}(x_{i};\tilde{\bm{\theta}}^{*}))(-y_{i})&\left[\sum_{j=1}^{M_{l}+1}\frac{\partial \tilde{f}(x_{i};\tilde{\bm{\theta}}^{*})}{\partial z^{(l)}_{j}}\frac{\partial z_{j}^{(l)}(x_{i})}{\partial b_{j}^{(l)}}{w_{j,M_{l-1}+1}^{(l)}}^{*}z_{M_{l-1}+1}^{(l-1)}(x_{i})\right]\\
&+\sum_{j=1}^{M_{l}+1}\left({w_{j,M_{l-1}+1}^{(l)}}^{*}\right)^{2}=0
\end{align}
Comparing Eq.~\eqref{eq::lemma3-6} and~\eqref{eq::lemma3-5}, we thus obtain 
\begin{equation*}
\sum_{j=1}^{M_{l}+1}\left({w_{j,M_{l-1}+1}^{(l)}}^{*}\right)^{2}=0,
\end{equation*}
and this indicates that 
$${w_{j,M_{l-1}+1}^{(l)}}^{*}=0, \quad \text{for }j=1,...,M_{l}+1\text{ and } l = 2,...,L.$$
\end{proof}

\subsection{Proof of Lemma~\ref{lemma::empirical-loss}}
\begin{figure}[t]
\centering
\includegraphics[width = 0.8\linewidth]{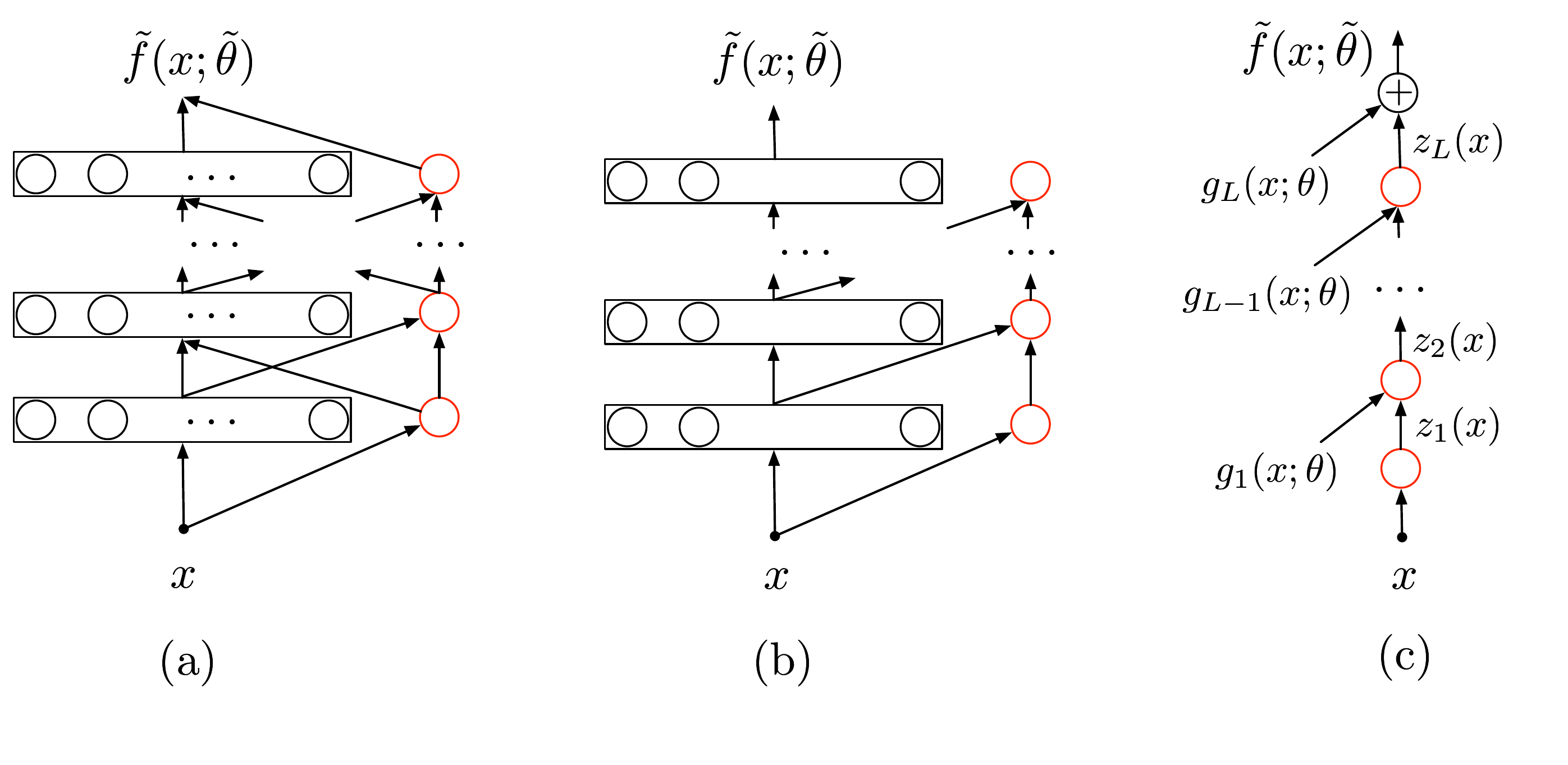}
\vspace{-0.4cm}
\caption{(a) The network architecture. (b) The network architecture at any local minimum, where all exponential neurons do not contribute to the final output of the neural network. (c) A simplification of the network architecture. Black circles represent the neuron activation in the original network $f(x;\bm{\theta})$ and red circles represent the exponential neuron.}
\label{fig::structure}
\end{figure}

\begin{proof}
 For simplification of the notation, we use the scalar $\alpha_{l-1}, l =2,...,L$ to denote the weight connecting the exponential neuron in the $l$-th layer and the exponential neuron in the $(l-1)$-th layer, i.e., $\alpha_{l-1}=w_{M_{l}+1, M_{l-1}+1}^{(l-1)}$. We use the scalar $\alpha_{L}$ to denote the weight connecting the exponential neuron in the $L$-th layer and the output layer, i.e., $\alpha_{L}=a_{M_{L}+1}$.  Since we have already proved  $a_{M_{L}+1}=0$ and $w_{M_{l}+1, M_{l-1}+1}^{(l-1)}=0$ for $l = 2,...,L,$ then $\alpha_{1}=...=\alpha_{L-1}=\alpha_{L}=0$. This indicates the network shown by Fig.~\ref{fig::structure} (a) is reduced to Fig.~\ref{fig::structure} (b) and further indicates that all exponential neurons do not affect the original neural network $f$. In addition, we use the function $g_{l-1}(x;\tilde{\bm{\theta}})$ to denote the weighted input from the other neurons except the exponential neuron in the $(l-1)$-th layer to the exponential neuron in the $l$-th layer. Besides, we use the function $z_{l}(x;\tilde{\bm{\theta}},\bm{\alpha})$ to denote the output of the exponential neuron in the $l$-th layer, where the vector $\bm{\alpha}$ is consisted of all $\alpha_{l}$s. 
Finally, we use the weight vector $\bm{w}$ and the scalar $b$ to denote the weight and bias of the exponential neuron in the first layer. 

Now we split the parameter vector $\tilde{\bm{\theta}}^{*}$ into two parts, i.e., $\tilde{\bm{\theta}}^{*}=(\bm{\theta}^{*}, \bm{\alpha}^{*},\bm{w}^{*}, b^{*})$. Since all exponential neurons do not affect the original neural network, then the output of the neural network $\tilde{f}(x;\tilde{\bm{\theta}^{*}})$ on the sample $x_{i}$ can be written as 
\begin{align}
\tilde{f}(x_{i};\tilde{\bm{\theta}}^{*}) = g_{L}(x_{i};{\bm{\theta}}^{*})+\alpha_{L}z_{L}(x_{i};\bm{\theta}^{*}, \bm{\alpha}^{*},\bm{w}^{*}, b^{*}),
\end{align}
where $g_{L}(x_{i};{\bm{\theta}})$ denotes the output coming from the last layer of the original neural network $f$ and $z_{L}(x_{i};\bm{\theta}, \bm{\alpha},\bm{w}, b)$ denotes the output coming from the exponential neuron in the $L$-th layer. From Lemma~\ref{lemma::multilayer-zero}, it follows that $\alpha_{L}^{*}=0$ and thus
\begin{align}
\tilde{f}(x_{i};\tilde{\bm{\theta}}^{*}) = g_{L}(x_{i};{\bm{\theta}}^{*}).
\end{align}
In addition, the output of the exponential neuron in the $l$-th layer is  
\begin{align*}
z_{l}(x_{i};\bm{\theta}^{*}, \bm{\alpha}^{*},\bm{w}^{*}, b^{*})=\exp\left(g_{l-1}(x_{i};{\bm{\theta}}^{*})+\alpha_{l-1}z_{l-1}(x_{i};\bm{\theta}^{*}, \bm{\alpha}^{*},\bm{w}^{*}, b^{*})\right), \quad l = 2,...,L
\end{align*}
and the output of the exponential neuron in the first layer is
\begin{equation*}
z_{l}(x_{i};\bm{w}^{*},b^{*})=\exp\left({\bm{w}^{*}}^{\top}x_{i} +b^{*}\right).
\end{equation*}
Since $\tilde{\bm{\theta}}^{*}=(\bm{\theta}^{*}, \bm{\alpha}^{*},\bm{w}^{*}, b^{*})$ is a local minimum of the empirical loss function, now we consider a small perturbation on parameters  $\bm{\alpha}$, $\bm{w}$ and $b$. Let $\tilde{\bm{\theta}}^{*}+\bm{\Delta}=(\bm{\theta}^{*}, \bm{\alpha}^{*}+\Delta\bm{\alpha}, \bm{w}^{*}+\Delta\bm{w}, b^{*}+\Delta b)$ denote the perturbed parameters with perturbations $\Delta\bm{\alpha}$, $\Delta\bm{w}$ and $\Delta b$ on parameters $\bm{\alpha}$, $\bm{w}$ and $b$, respectively. Thus, by the definition of the local minimum, there exists a positive number $\varepsilon_{0}<1$ such that 
$$\tilde{L}_{n}(\tilde{\bm{\theta}}^{*}+\bm{\Delta})\ge \tilde{L}_{n}(\tilde{\bm{\theta}}^{*})$$
holds for all $\bm{\Delta}:\|\bm{\Delta}\|_{2}\le \varepsilon$. We note that since we do not perturb parameters $\bm{\theta}^{*}$ in the original neural network, then  the value of $g_{l}(x;\bm{\theta}^{*})$ does not change under parameters $\tilde{\bm{\theta}}^{*}+\bm{\Delta}$.

Now we consider the value of the loss function under parameters $\tilde{\bm{\theta}}^{*}+\bm{\Delta}$,
\begin{align*}
\tilde{L}_{n}(\tilde{\bm{\theta}}^{*}+\bm{\Delta})&=\sum_{i=1}^{n}\ell\left(-y_{i}\tilde{f}(x_{i};\tilde{\bm{\theta}}^{*}+\bm{\Delta})\right)+\frac{\lambda\|\Delta\bm{\alpha}\|^{2L}_{2L}}{2}\\
&=\sum_{i=1}^{n}\ell\left(-y_{i}g_{L}(x_{i};{\bm{\theta}}^{*})-y_{i}\Delta\alpha_{L}z_{L}(x_{i};\tilde{\bm{\theta}}^{*}+\bm{\Delta})\right)+\frac{\lambda\|\Delta\bm{\alpha}\|^{2L}_{2L}}{2}.
\end{align*}
By Taylor expansion, there exists a constant $\xi\in(0,1)$ such that 
\begin{align*}
\tilde{L}_{n}(\tilde{\bm{\theta}}^{*}+\bm{\Delta})&=\sum_{i=1}^{n}\ell\left(-y_{i}g_{L}(x_{i};{\bm{\theta}}^{*})-y_{i}\Delta\alpha_{L}z_{L}(x_{i};\tilde{\bm{\theta}}^{*}+\bm{\Delta})\right)+\frac{\lambda\|\Delta\bm{\alpha}\|^{2L}_{2L}}{2}\\
&=\sum_{i=1}^{n}\ell\left(-y_{i}g_{L}(x_{i};{\bm{\theta}}^{*})\right)+\sum_{i=1}^{n}\ell'\left(-y_{i}g_{L}(x_{i};{\bm{\theta}}^{*})\right)(-y_{i})\Delta\alpha_{L}z_{L}(x_{i};\tilde{\bm{\theta}}^{*}+\bm{\Delta})\\
&\quad+\sum_{i=1}^{n}\ell''\left(-y_{i}g_{L}(x_{i};{\bm{\theta}}^{*})-\xi y_{i}\Delta\alpha_{L}z_{L}(x_{i};\tilde{\bm{\theta}}^{*}+\bm{\Delta})\right)\left(\Delta\alpha_{L}\right)^{2}z^{2}_{L}(x_{i};\tilde{\bm{\theta}}^{*}+\bm{\Delta})\\
&\quad+\frac{\lambda\|\Delta\bm{\alpha}\|^{2L}_{2L}}{2}
\end{align*}
Given the parameters at the local minimum $\tilde{\bm{\theta}}^{*}$ and the dataset $\mathcal{D}=\{(x_{i},y_{i})\}_{i=1}^{n}$, $\ell''$ is a continuous function on $(\bm{\Delta},\xi)$ and  $z_{L}$ is a continuous function on $\bm{\Delta}$. Thus, on the bounded region, 
$\{(\bm{\Delta},\xi): \xi\in(0,1),\|\bm{\Delta}\|_{2}\le 1\}$, there exists a constant $C_{1}(\tilde{\bm{\theta}},\mathcal{D})$ depending on $\tilde{\bm{\theta}},\mathcal{D}$ such that  
\begin{equation}
\sum_{i=1}^{n}\ell''\left(-y_{i}g_{L}(x_{i};{\bm{\theta}}^{*})-\xi y_{i}\Delta\alpha_{L}z_{L}(x_{i};\tilde{\bm{\theta}}^{*}+\bm{\Delta})\right)z^{2}_{L}(x_{i};\tilde{\bm{\theta}}^{*}+\bm{\Delta})\le C_{1}(\tilde{\bm{\theta}}^{*},\mathcal{D}).
\end{equation}
This indicates that 
\begin{align*}
\tilde{L}_{n}(\tilde{\bm{\theta}}^{*},\bm{\Delta})&\le\sum_{i=1}^{n}\ell\left(-y_{i}g_{L}(x_{i};\tilde{\bm{\theta}}^{*})\right)+\Delta\alpha_{L}\sum_{i=1}^{n}\ell'\left(-y_{i}g_{L}(x_{i};{\bm{\theta}}^{*})\right)(-y_{i})z_{L}(x_{i};\tilde{\bm{\theta}}^{*}+\bm{\Delta})\\
&\quad+C_{1}(\tilde{\bm{\theta}}^{*},\mathcal{D})\left(\Delta\alpha_{L}\right)^{2}+\frac{\lambda\|\Delta\bm{\alpha}\|^{2L}_{2L}}{2}.
\end{align*}

Now we consider $z_{L}(x_{i};\tilde{\bm{\theta}}^{*}+\bm{\Delta})$ and introduce a new parameter $\lambda$ by setting 
\begin{equation*}
z_{l}(x_{i};\tilde{\bm{\theta}}^{*}+\bm{\Delta},\lambda)=\exp\left(g_{l-1}(x_{i};{\bm{\theta}}^{*})+\lambda\Delta\alpha_{l-1}z_{l-1}(x_{i};\tilde{\bm{\theta}}^{*}+\bm{\Delta}, \lambda)\right), \quad l = 2,...,L
\end{equation*}
\begin{equation*}
z_{1}(x_{i};\tilde{\bm{\theta}}^{*}+\bm{\Delta},\lambda)=\exp\left({\bm{w}^{*}}^{\top}x_{i} +b^{*}+\Delta\bm{{w}}^{\top}x_{i}+\Delta b\right)
\end{equation*}
Now we set $\Delta\alpha_{1}=...=\Delta\alpha_{L-1}=\exp(-1/\varepsilon)$, $\alpha_{L}=\exp(-(L+1)/\varepsilon)$ and $(\Delta\bm{w}, \Delta b )=(\varepsilon\bm{u}, \varepsilon v)$ for any $\bm{u}, v$ satisfying $\|\bm{u}\|^{2}_{2}+v^{2}=1.$ Obviously, $\left\|\bm{\Delta}\right\|_{2}\rightarrow 0$ as $\varepsilon\rightarrow 0$. This indicates that there exists a $\varepsilon_{1}<1$ such that for all $\varepsilon<\varepsilon_{1}$, the inequality
$$\tilde{L}_{n}(\tilde{\bm{\theta}}^{*}+\bm{\Delta})\ge \tilde{L}_{n}(\tilde{\bm{\theta}}^{*})$$
always holds. 

Now, for a given $\varepsilon$, we can view $z_{L}(x_{i};\bm{\theta}^{*},\varepsilon,\lambda)$ as a function of $\lambda$ and  expand the function at the point $\lambda =0$. Therefore, for each sample $(x_{i}, y_{i})$, there exists a scalar $c_{i}\in(0, 1)$ such that 
\begin{align*}
z_{L}(x_{i};\tilde{\bm{\theta}}^{*}+\bm{\Delta})&= z_{L}(x_{i};\tilde{\bm{\theta}}^{*},\varepsilon, \lambda)\left.\right|_{\lambda=1}\\
&=\sum_{k=0}^{L-1}\frac{1}{k!}\left.\frac{d^{k}z_{L}(x_{i};\bm{\Delta},\lambda)}{d\lambda^{k}}\right|_{\lambda=0}+\left.\frac{1}{L!}\frac{d^{L}z_{L}(x_{i};\bm{\Delta},\lambda)}{d\lambda^{L}}\right|_{\lambda=c_{i}}\\
\end{align*}

We first prove the following claim
\begin{claim}\label{claim::5}
For $l=2,...,L$, 
\begin{equation}\label{eq::claim5-1}
\frac{dz_{l}(x_{i};\tilde{\bm{\theta}}^{*},\varepsilon,\lambda)}{d\lambda}=\sum_{k=1}^{l-1}\lambda^{l-1-k}e^{-\frac{l-k}{\varepsilon}}\prod_{j=k}^{l}z_{k}(x_{i};\tilde{\bm{\theta}}^{*},\varepsilon,\lambda).
\end{equation}
\end{claim}
\begin{proof}
Since
$$z_{1}(x_{i})=\exp({\bm{w}^{*}}^{\top}x_{i}+b^{*}),$$
then 
$$\frac{d z_{1}(x_{i})}{d\lambda}=0.$$
Furthermore,  for $l=2,...,L$
$$z_{l}(x_{i};\tilde{\bm{\theta}}^{*},\varepsilon,\lambda)=\exp\left(g_{l-1}(x_{i};\bm{\theta}^{*})+\lambda e^{-1/\varepsilon}z_{l-1}(x_{i};\tilde{\bm{\theta}}^{*},\varepsilon,\lambda)\right),$$
we should have 
\begin{align*}
\frac{dz_{l}}{d\lambda}=z_{l}\left(e^{-1/\varepsilon}z_{l-1}+\lambda e^{-1/\varepsilon}\frac{dz_{l-1}}{d\lambda}\right)=z_{l}e^{-1/\varepsilon}\left(z_{l-1}+\lambda\frac{dz_{l-1}}{d\lambda}\right).
\end{align*}
Therefore, the base hypothesis $l=2$ holds directly by
\begin{align*}
\frac{dz_{2}}{d\lambda}=e^{-1/\varepsilon}z_{2}z_{1}. 
\end{align*}
Now the inductive hypothesis is Eq.~\eqref{eq::claim5-1} holds when $l=s$. Then for $l=s+1$, we have 
\begin{align*}
\frac{dz_{s+1}}{d\lambda}&=e^{-1/\varepsilon}z_{s+1}z_{s}+\lambda e^{-1/\e}z_{s+1}\sum_{k=1}^{s-1}\lambda^{s-1-k}e^{-\frac{s-k}{\e}}\prod_{j=k}^{s}z_{j}\\
&=e^{-1/\varepsilon}z_{s+1}z_{s}+\sum_{k=1}^{s-1}\lambda^{s-k}e^{-\frac{s+1-k}{\e}}\prod_{j=k}^{s+1}z_{j}\\
&=\sum_{k=1}^{s}\lambda^{s-k}e^{-\frac{s+1-k}{\e}}\prod_{j=k}^{s+1}z_{j}
\end{align*} 
and this indicates Eq.~\eqref{eq::claim5-1} holds for $l=s+1$. Thus, we finished the induction. 
\end{proof}

\begin{claim}\label{claim::6}
For each integer $m\ge 1$, there exists functions $\beta(i_{1},...,i_{L}):\mathbb{N}^{L}\rightarrow \mathbb{N}$, $\gamma_{1}(i_{1},...,i_{L}):\mathbb{N}^{L}\rightarrow \mathbb{N}$ and $\gamma_{2}(i_{1},...,i_{L}):\mathbb{N}^{L}\rightarrow \mathbb{N}$ satisfying 
\begin{align*}
\gamma_{2}(i_{1},...,i_{L})= i_{1}+...+i_{L}-1=\gamma_{1}(i_{1},...,i_{L})+m
\end{align*}
for any $i_{1},...,i_{L}\in\mathbb{N}$ such that 
\begin{equation}\label{eq::claim6-gradient}
\frac{d^{m}z_{L}}{d\lambda^{m}}=\sum_{i_{1},...,i_{L}\ge 0}\beta(i_{1},...,i_{L})\lambda^{\gamma_{1}(i_{1},...,i_{L})}e^{-\frac{\gamma_{2}(i_{1},...,i_{L})}{\e}}\prod_{j=1}^{L}z_{j}^{i_{j}}.
\end{equation}
\end{claim}
\begin{proof}
Now we first check the base hypothesis where $m=1$. By claim~\ref{claim::5}, we have proved that 
$$\frac{dz_{L}}{d\lambda}=\sum_{k=1}^{L-1}\lambda^{L-1-k}e^{-\frac{L-k}{\varepsilon}}\prod_{j=k}^{L}z_{j}.$$
Then we can observe that (1) the coefficient of each term in the summation, the power of $\lambda$ and the power of $e^{-1/\e}$ are all natural numbers, therefore, functions $\beta, \gamma_{1},\gamma_{2}:\mathbb{N}^{L}\rightarrow \mathbb{N}$ exist; (2) $\gamma_{2}=L-k= L-k-1+1=\gamma_{1}+1$ and $\gamma_{2}=L-k= i_{1}+...+i_{L}-1=L-k+1-1$. Thus, we have established the base hypothesis. 

Now the inductive hypothesis is the claim holds under $m$ and we next check the claim under $m+1$. 
Since the claim holds under $m$, then we have 
\begin{equation}
\frac{d^{m}z_{L}}{d\lambda^{m}}=\sum_{i_{1},...,i_{L}\ge 0, i_{1}+...+i_{L}\ge m+1}\beta(i_{1},...,i_{L})\lambda^{\gamma_{1}(i_{1},...,i_{L})}e^{-\frac{\gamma_{2}(i_{1},...,i_{L})}{\e}}\prod_{j=1}^{L}z_{j}^{i_{j}},
\end{equation}
and $\gamma_{1}$ and $\gamma_{2}$ satisfy $$\gamma_{2}(i_{1},...,i_{L})= \gamma_{1}(i_{1},...,i_{L})+m= i_{1}+...+i_{L}-1.$$
Thus 
\begin{align*}
\frac{d^{m+1}z_{L}}{d\lambda^{m+1}}&=\frac{d}{d\lambda}\left[\sum_{i_{1},...,i_{L}\ge 0, i_{1}+...+i_{L}\ge m+1}\beta(i_{1},...,i_{L})\lambda^{\gamma_{1}(i_{1},...,i_{L})}e^{-\frac{\gamma_{2}(i_{1},...,i_{L})}{\e}}\prod_{j=1}^{L}z_{j}^{i_{j}}\right]\\
&=\sum_{i_{1},...,i_{L}\ge 0, i_{1}+...+i_{L}\ge m+1}\beta(i_{1},...,i_{L})\frac{d}{d\lambda}\left[\lambda^{\gamma_{1}(i_{1},...,i_{L})}e^{-\frac{\gamma_{2}(i_{1},...,i_{L})}{\e}}\prod_{j=1}^{L}z_{j}^{i_{j}}\right]
\end{align*}
We first consider the case where $\gamma_{1}(i_{1},...,i_{L})=0$, 
\begin{align*}
\frac{d}{d\lambda}\left[e^{-\frac{\gamma_{2}(i_{1},...,i_{L})}{\e}}\prod_{j=1}^{L}z_{j}^{i_{j}}\right]=e^{-\frac{\gamma_{2}(i_{1},...,i_{L})}{\e}}\frac{d}{d\lambda}\left[\prod_{j=1}^{L}z_{j}^{i_{j}}\right],
\end{align*} 
where 
\begin{align}
\frac{d}{d\lambda}\prod_{j=1}^{L}z_{j}^{i_{j}}&=\left[\prod_{j=1}^{L}z_{j}^{i_{j}}\right]\left[\sum_{k=1}^{L}i_{k}z_{k}^{-1}\frac{dz_{k}}{d\lambda}\right]\notag\\
&=\left[\prod_{j=1}^{L}z_{j}^{i_{j}}\right]\left[\sum_{k=1}^{L}i_{k}\sum_{j=1}^{k-1}\lambda^{k-1-j}e^{-\frac{k-j}{\e}}\prod_{r=j}^{k-1}z_{r}\right].\label{eq::claim6-2}
\end{align}
First, it is easy to see that the coefficient of each term in~\eqref{eq::claim6-2} is a natural number and the power of $\lambda$ and $e^{-1/\e}$ are all natural number as well.  Second, since $i_{1}+...+i_{L}= \gamma_{2}+1$ and for the term $\lambda^{k-1-j}e^{-\frac{k-j}{\e}}\prod_{r=j}^{k-1}z_{r}$, the total degree of $\prod_{r=j}^{k-1}z_{r}$ is  $k-j$ and always equals to the power of $e^{-1/\e}$, then for each term in $e^{-\frac{\gamma_{2}(i_{1},...,i_{L})}{\e}}\left[\prod_{j=1}^{L}z_{j}^{i_{j}}\right]\left[\sum_{k=1}^{L}i_{k}\sum_{j=1}^{k-1}\lambda^{k-1-j}e^{-\frac{k-j}{\e}}\prod_{r=j}^{k-1}z_{r}\right]$ the degree of $z_{l}$ product always exceeds the power of $e^{-1/\e}$ by one and the power of $e^{-1/\e}$ always exceeds the power of $\lambda$ by $m+1$. Therefore, the hypothesis holds under $m+1$ for the term $\frac{d}{d\lambda}\left[e^{-\frac{\gamma_{2}(i_{1},...,i_{L})}{\e}}\prod_{j=1}^{L}z_{j}^{i_{j}}\right]$.

Next we consider the case where $\gamma_{1}(i_{1},...,i_{L})\ge 1$, 
\begin{align*}
\frac{d}{d\lambda}\left[\lambda^{\gamma_{1}(i_{1},...,i_{L})}e^{-\frac{\gamma_{2}(i_{1},...,i_{L})}{\e}}\prod_{j=1}^{L}z_{j}^{i_{j}}\right]&=\gamma_{1}(i_{1},...,i_{L})\lambda^{\gamma_{1}(i_{1},...,i_{L})-1}e^{-\frac{\gamma_{2}(i_{1},...,i_{L})}{\e}}\prod_{j=1}^{L}z_{j}^{i_{j}}\\
&\quad+\lambda^{\gamma_{1}(i_{1},...,i_{L})}e^{-\frac{\gamma_{2}(i_{1},...,i_{L})}{\e}}\frac{d}{d\lambda}\left[\prod_{j=1}^{L}z_{j}^{i_{j}}\right],
\end{align*}

Now we check whether the hypothesis holds for the term  
$$\gamma_{1}(i_{1},...,i_{L})\lambda^{\gamma_{1}(i_{1},...,i_{L})-1}e^{-\frac{\gamma_{2}(i_{1},...,i_{L})}{\e}}\prod_{j=1}^{L}z_{j}^{i_{j}}.$$
Since $\gamma_{1}$ is a natural number and $\gamma_{1}\ge 1$, and then $\gamma_{1}-1$ is still a natural number. Next, the power of $e^{-1/\e}$ exceeds the power of $\lambda$ by $m+1$. Furthermore, $i_{1}+...+i_{L}= \gamma_{2}(i_{1},...,i_{L})+1$ still holds by the hypothesis that the claim holds under $m$. In addition, since $\gamma_{1}\ge 1$, then by the hypothesis, we have $i_{1}+...+i_{L}-1= \gamma_{1}+m$ and thus have $i_{1}+...+i_{L}-1= \gamma_{1}-1+m+1$. Thus, the hypothesis under $m+1$ holds for the term $\gamma_{1}(i_{1},...,i_{L})\lambda^{\gamma_{1}(i_{1},...,i_{L})-1}e^{-\frac{\gamma_{2}(i_{1},...,i_{L})}{\e}}\prod_{j=1}^{L}z_{j}^{i_{j}}.$

Finally, we check whether the hypothesis holds in the term 
\begin{align*}
\lambda^{\gamma_{1}(i_{1},...,i_{L})}&e^{-\frac{\gamma_{2}(i_{1},...,i_{L})}{\e}}\frac{d}{d\lambda}\left[\prod_{j=1}^{L}z_{j}^{i_{j}}\right]\\
&=\lambda^{\gamma_{1}(i_{1},...,i_{L})}e^{-\frac{\gamma_{2}(i_{1},...,i_{L})}{\e}}\left[\prod_{j=1}^{L}z_{j}^{i_{j}}\right]\left[\sum_{k=1}^{L}i_{k}\sum_{j=1}^{k-1}\lambda^{k-1-j}e^{-\frac{k-j}{\e}}\prod_{r=j}^{k-1}z_{r}\right]
\end{align*}
Since $i_{k}\in\mathbb{N}$ holds for all $k$, then the coefficient of each term in  $\frac{d}{d\lambda}\left[\lambda^{\gamma_{1}(i_{1},...,i_{L})}e^{-\frac{\gamma_{2}(i_{1},...,i_{L})}{\e}}\prod_{j=1}^{L}z_{j}^{i_{j}}\right]$ is a natural number. Further, the power of $\lambda$ and $e^{-1/\e}$ are all natural numbers. Second, since $\gamma_{2}(i_{1},...,i_{L})= i_{1}+...+i_{L}-1$ and for the term $\lambda^{k-1-j}e^{-\frac{k-j}{\e}}\prod_{r=j}^{k-1}z_{r}$, the power of $e^{-1/\e}$ is $k-j$ and thus equals to the degree of the product term $\prod_{r=j}^{k-1}z_{r}$. Third, since for the term $\lambda^{k-1-j}e^{-\frac{k-j}{\e}}\prod_{r=j}^{k-1}z_{r}$, the degree of $\prod_{r=j}^{k-1}z_{r}$ is $k-j$ always exceed the power of $\lambda$ by one and $i_{1}+...+i_{L}= \gamma_{1}+m+1$, then each term in the expression $\lambda^{\gamma_{1}(i_{1},...,i_{L})}e^{-\frac{\gamma_{2}(i_{1},...,i_{L})}{\e}}\frac{d}{d\lambda}\left[\prod_{j=1}^{L}z_{j}^{i_{j}}\right]$, the degree of $z$s product always exceed the power of $\lambda$ by $m+2$. Thus, the hypothesis holds for the term $\lambda^{\gamma_{1}(i_{1},...,i_{L})}e^{-\frac{\gamma_{2}(i_{1},...,i_{L})}{\e}}\frac{d}{d\lambda}\left[\prod_{j=1}^{L}z_{j}^{i_{j}}\right]$.

Therefore, we complete the induction. 
\end{proof}

\begin{claim} There exists a series of function $\rho_{k}$, $k=0,...,L-1$ and function $\eta$  such that
\begin{equation*}
\left.\sum_{l=1}^{L-1}\frac{d^{l}z_{L}}{d{\lambda}^{l}}\right|_{\lambda=0}= \sum_{k=0}^{L-1}\rho_{k}(\tilde{\bm{\theta}}^{*}, x_{i})e^{-k/\e}+\eta(\tilde{\bm{\theta}}^{*}, x_{i})e^{-(L-1)/\e}\exp(\e \bm{u}^{\top}x_{i}+\e v),
\end{equation*}
where  the function $\eta$ satisfies $\eta(\tilde{\bm{\theta}},x)>0$ holds for all $x\in\mathbb{R}^{d}$ and any parameter $\tilde{\bm{\theta}}$.
\end{claim}
\begin{proof}
By Claim~\ref{claim::6}, $(L-1)$-th order derivative can be rewritten as 
 \begin{equation}
\frac{d^{L-1}z_{L}}{d\lambda^{L-1}}=\sum_{i_{1},...,i_{L}\ge 0}\beta(i_{1},...,i_{L})\lambda^{\gamma_{1}(i_{1},...,i_{L})}e^{-\frac{\gamma_{2}(i_{1},...,i_{L})}{\e}}\prod_{j=1}^{L}z_{j}^{i_{j}},
\end{equation}
for some function $\beta, \gamma_{1}$ and $\gamma_{2}$ satisfying 
\begin{align*}
\gamma_{2}(i_{1},...,i_{L})= i_{1}+...+i_{L}-1=\gamma_{1}(i_{1},...,i_{L})+L-1.
\end{align*}
Now we calculate the value of $\frac{d^{L-1}z_{L}}{d\lambda^{L-1}}$ at the point $\lambda=0$. Then
\begin{align*}
\left.\frac{d^{L-1}z_{L}}{d\lambda^{L-1}}\right|_{\lambda=0}&=\sum_{i_{1},...,i_{L}\ge 0, i_{1}+...+i_{L}=L}\beta(i_{1},...,i_{L})e^{-\frac{L}{\e}}\left.\prod_{j=1}^{L}z_{j}^{i_{j}}\right|_{\lambda=0}
\end{align*}
By the property of the chain's rule and the property of exponential neuron, we can easily see that if $i_{L}\ge 1$, then $i_{1}\ge 1,...,i_{L-1}\ge 1$. This indicates that if $i_{L}\ge 1$ and $i_{1},...,i_{L}$ satisfy $i_{1}+...+i_{L}= L$, then $i_{1}=...=i_{L}=1$. Therefore, we have 
\begin{align*}
&\left.\frac{d^{L-1}z_{L}}{d\lambda^{L-1}}\right|_{\lambda=0}=\beta(1,...,1)e^{-(L-1)/\varepsilon}\exp\left(\sum_{l=1}^{L-1}g_{l}(x_{i};\bm{\theta}^{*})+{\bm{w}^{*}}^{\top}x_{i}+b^{*}+\e\bm{u}^{\top}x_{i}+\e v\right)\\
&+\sum_{i_{1},...,i_{L}\ge 0, i_{1}+...+i_{L}=L}\beta(i_{1},...,i_{L})e^{-\frac{L-1}{\e}}\exp\left(\sum_{l=1}^{L-1}i_{l+1}g_{l}(x_{i};\bm{\theta}^{*})\right)
\end{align*}
Now we define 
\begin{align*}
\eta(x_{i};\bm{\theta}^{*})&\triangleq\beta(1,...,1)\exp\left(\sum_{l=1}^{L-1}g_{l}(x_{i};\bm{\theta}^{*})+{\bm{w}^{*}}^{\top}x_{i}+b^{*}\right),\\
\rho_{L-1}(x_{i};\bm{\theta}^{*})&\triangleq\sum_{i_{1},...,i_{L}\ge 0, i_{1}+...+i_{L}=L}\beta(i_{1},...,i_{L})\exp\left(\sum_{l=1}^{L-1}i_{l+1}g_{l}(x_{i};\bm{\theta}^{*})\right),
\end{align*}
and thus 
\begin{align*}
\left.\frac{d^{L-1}z_{L}}{d\lambda^{L-1}}\right|_{\lambda=0}&=e^{-(L-1)/\e}\eta(x_{i};\bm{\theta}^{*})\exp(\e\bm{u}^{\top}x_{i}+\e v)+\rho_{L-1}(x_{i};\bm{\theta}^{*})e^{-(L-1)/\e}.
\end{align*}

Furthermore, we can rewrite $d^{l}z_{L}/d\lambda^{l}$ as 
 \begin{equation}
\frac{d^{l}z_{L}}{d\lambda^{l}}=\sum_{i_{1},...,i_{L}\ge 0}\beta_{l}(i_{1},...,i_{L})\lambda^{\gamma_{1,l}(i_{1},...,i_{L})}e^{-\frac{\gamma_{2,l}(i_{1},...,i_{L})}{\e}}\prod_{j=1}^{L}z_{j}^{i_{j}},
\end{equation}
for some function $\beta_{l}, \gamma_{1,l}$ and $\gamma_{2,l}$ satisfying 
\begin{align*}
\gamma_{2,l}(i_{1},...,i_{L})= i_{1}+...+i_{L}-1=\gamma_{1,l}(i_{1},...,i_{L})+l.
\end{align*}
Therefore, the value of $d^{l}z_{L}/d\lambda^{l}$ at $\lambda =0$ is
\begin{align*}
\left.\frac{d^{l}z_{L}}{d\lambda^{l}}\right|_{\lambda=0}&=\sum_{i_{1},...,i_{L}\ge 0,i_{1}+...+i_{L}=l+1}\beta_{l}(i_{1},...,i_{L})e^{-\frac{l}{\e}}\prod_{j=1}^{L}z_{j}^{i_{j}}\\
&=\sum_{i_{1},...,i_{L}\ge 0,i_{1}+...+i_{L}=l+1}\beta_{l}(i_{1},...,i_{L})e^{-\frac{l}{\e}}\exp\left(\sum_{j=1}^{L-1}i_{j+1}g_{j}(x_{i};\bm{\theta}^{*})+i_{1}\e\bm{u}^{\top}x_{i}+i_{1}\e v\right).
\end{align*}
Furthermore, since $i_{1}\ge 1$ indicates that $i_{1}+...+i_{L}\ge L$ and $l< L-1$, then this we have for all $l<L-1$, $i_{1}=0$ and 
\begin{align*}
\left.\frac{d^{l}z_{L}}{d\lambda^{l}}\right|_{\lambda=0}&=\sum_{i_{1},...,i_{L}\ge 0,i_{1}+...+i_{L}=l+1}\beta_{l}(i_{1},...,i_{L})e^{-\frac{l}{\e}}\exp\left(\sum_{j=1}^{L-1}i_{j+1}g_{j}(x_{i};\bm{\theta}^{*})\right)\\
&\triangleq\rho_{l}(x_{i};\bm{\theta}^{*})e^{-l/\e}.
\end{align*}
Therefore, we have 
\begin{equation*}
\left.\sum_{l=1}^{L-1}\frac{d^{l}z_{L}}{d{\lambda}^{l}}\right|_{\lambda=0}= \sum_{k=0}^{L-1}\rho_{k}(\tilde{\bm{\theta}}^{*}, x_{i})e^{-k/\e}+\eta(\tilde{\bm{\theta}}^{*}, x_{i})e^{-(L-1)/\e}\exp(\e \bm{u}^{\top}x_{i}+\e v),
\end{equation*}
where  the function $\eta$ satisfies $\eta(\tilde{\bm{\theta}},x)>0$ holds for all $x\in\mathbb{R}^{d}$ and any parameter $\tilde{\bm{\theta}}$.
\end{proof}

Therefore, we have 
\begin{align*}
z_{L}(x_{i};\tilde{\bm{\theta}}^{*}+\bm{\Delta})&=\sum_{k=0}^{L-1}\frac{1}{k!}\left.\frac{d^{k}z_{L}(x_{i};\bm{\Delta},\lambda)}{d\lambda^{k}}\right|_{\lambda=0}+\left.\frac{1}{L!}\frac{d^{L}z_{L}(x_{i};\bm{\Delta},\lambda)}{d\lambda^{L}}\right|_{\lambda=c_{i}}\\
&=\sum_{k=0}^{L-1}\rho_{k}(\tilde{\bm{\theta}}^{*}, x_{i})e^{-k/\e}+\eta(\tilde{\bm{\theta}}^{*}, x_{i})e^{-(L-1)/\e}\exp(\e \bm{u}^{\top}x_{i}+\e v)+\left.\frac{1}{L!}\frac{d^{L}z_{L}(x_{i};\tilde{\bm{\theta}}^{*},\e,\lambda)}{d\lambda^{L}}\right|_{\lambda=c_{i}}.
\end{align*}
Since $\frac{d^{L}z_{L}(x_{i};\bm{\Delta},\lambda)}{d\lambda^{L}}$ can be rewritten as 
\begin{equation}
\frac{d^{L}z_{L}}{d\lambda^{L}}=\sum_{i_{1},...,i_{L}\ge 0}\beta_{L}(i_{1},...,i_{L})\lambda^{\gamma_{1,L}(i_{1},...,i_{L})}e^{-\frac{\gamma_{2,L}(i_{1},...,i_{L})}{\e}}\prod_{j=1}^{L}z_{j}^{i_{j}}
\end{equation}
for some function $\beta_{L}, \gamma_{1,L},\gamma_{2,L}$ satisfying 
$$\gamma_{1,L}(i_{1},...,i_{L})+L=\gamma_{2,L}(i_{1},...,i_{L}).$$
Since the value of  function $\gamma_{1,L}(i_{1},...,i_{L})$ is always a natural number, then $\gamma_{2,L}(i_{1},...,i_{L})\ge L$. Since number of terms in $\frac{d^{L}z_{L}}{d\lambda^{L}}$ is bounded, $c_{i}\in(0,1),\e\in(0,1)$ and function $z_{j}(x_{i};\tilde{\bm{\theta}}^{*},\e,\lambda)$ is a continuous function in $\e$ and $\lambda$, then $\frac{d^{L}z_{L}(x_{i};\bm{\Delta},\lambda)}{d\lambda^{L}}$ can rewritten as 
$$\left.\frac{d^{L}z_{L}(x_{i};\bm{\Delta},\lambda)}{d\lambda^{L}}\right|_{\lambda=c_{i}}=e^{-L/\e}R(\e,c_{i};x_{i},\bm{\theta}^{*})$$
where the function $R(\e,c_{i};x_{i},\bm{\theta}^{*})$ is a continuous function on $(\e,c_{i})$.
Therefore, we have 
\begin{align*}
L(\tilde{\bm{\theta}}^{*},\bm{\Delta})&\le L(\tilde{\bm{\theta}}^{*})+\sgn(\Delta\alpha_{L})e^{-L/\e}\sum_{i=1}^{n}\ell'\left(-y_{i}g_{L}(x_{i};{\bm{\theta}}^{*})\right)(-y_{i})z_{L}(x_{i};\tilde{\bm{\theta}}^{*}+\bm{\Delta})+C_{1}(\tilde{\bm{\theta}}^{*},\mathcal{D})e^{-2L/\e}\\
&=L(\tilde{\bm{\theta}}^{*})+\sgn(\Delta\alpha_{L})e^{-L/\e}\sum_{k=0}^{L-1}e^{-k/\e}\sum_{i=1}^{n}\ell'\left(-y_{i}g_{L}(x_{i};{\bm{\theta}}^{*})\right)(-y_{i})\rho_{k}(x_{i},\tilde{\bm{\theta}}^{*})\\
&\quad+\sgn(\Delta\alpha_{L})e^{-(2L-1)/\e}\sum_{i=1}^{n}\ell'\left(-y_{i}g_{L}(x_{i};{\bm{\theta}}^{*})\right)(-y_{i})\eta(\tilde{\bm{\theta}}^{*}, x_{i})\exp(\e \bm{u}^{\top}x_{i}+\e v)\\
&\quad+e^{-2L/\e}\sum_{i=1}^{n}\ell'\left(-y_{i}g_{L}(x_{i};{\bm{\theta}}^{*})\right)(-y_{i})R(\e,c_{i};x_{i},\bm{\theta}^{*})+C_{1}(\tilde{\bm{\theta}}^{*},\mathcal{D})e^{-2L/\e}
\end{align*}
Using the same analysis method we have used in the proof of Theorem~\ref{thm::multi-exp}, we have 
$$\sum_{i=1}^{n}\ell'\left(-y_{i}g_{L}(x_{i};{\bm{\theta}}^{*})\right)(-y_{i})\eta(\tilde{\bm{\theta}}^{*}, x_{i})(\bm{u}^{\top}x_{i}+ v)^{p}=0$$
holds for all $(\bm{u},v):\|\bm{u}\|_{2}^{2}+v^{2}=1$ and all integer $p\ge 0$.
Furthermore, since 
$$\eta(\tilde{\bm{\theta}}, x_{i})=\beta(1,...,1)\exp\left(\sum_{l=1}^{L-1}g_{l}(x_{i};\bm{\theta})+\bm{w}^{\top}x_{i}+b\right)>0$$
holds for all $\bm{\theta}, x_{i}$ and 
$$\tilde{f}(x;\tilde{\bm{\theta}}^{*})=g_{L}(x;\bm{\theta}^{*}),$$
then there exists a function $\eta:\mathbb{R}^{d}\times\mathbb{R}^{|\tilde{\bm{\theta}}|}\rightarrow\mathbb{R}^{+}$ such that for every integer $p\ge 0$, the equation 
\begin{equation}
\sum_{i=1}^{n}\ell'\left(-y_{i}g_{L}(x_{i};\tilde{\bm{\theta}}^{*})\right)(-y_{i})\eta(x_{i},\tilde{\bm{\theta}}^{*})(\bm{u}^{\top}x_{i}+v)^{p}=0
\end{equation}
holds for any vector $\bm{u}$ and  scalar $v$ satisfying $\|\bm{u}\|_{2}^{2}+v^{2}=1$.
\end{proof}

\subsection{Proof of Theorem~\ref{thm::multi-exp}}
	
\begin{proof}
For every dataset $\mathcal{D}$ satisfying Assumption~\ref{assump::realizability}, by the Lagrangian interpolating polynomial, there always exists a polynomial  $P(x)=\sum_{j}c_{j}\pi_{j}(x)$ defined on $\mathbb{R}^{d}$ such that it can correctly classify all samples in the dataset with margin at least one, i.e., $y_{i}P(x_{i})\ge 1,\forall i\in[n]$, where $\pi_{j}$ denotes the $j$-th monomial in the polynomial $P(x)$. 
Therefore, from Lemma~\ref{lemma::tensor}, \ref{lemma::multilayer-zero} and \ref{lemma::empirical-loss}, it follows that 
\begin{align*}
\sum_{i=1}^{n}\ell'(-y_{i}\tilde{f}(x_{i};\tilde{\bm{\theta}}^{*}))\eta(x_{i};\tilde{\bm{\theta}}^{*})y_{i}P(x_{i})=\sum_{j}c_{j}\sum_{i=1}^{n}\ell'(-y_{i}\tilde{f}(x_{i};\tilde{\bm{\theta}}^{*}))y_{i}\eta(x_{i};\tilde{\bm{\theta}}^{*})\pi_{j}(x_{i})=0.
\end{align*}
Since $y_{i}P(x_{i})\ge 1$ and $e^{{\bm{w}^{*}}^{\top}x_{i}+b^{*}}>0$ hold for  $\forall i\in[n]$ and the loss function $\ell$ is a non-decreasing function, i.e., $\ell'(z)\ge 0,\forall z\in\mathbb{R}$, then $\ell'(-y_{i}\tilde{f}(x_{i};\tilde{\bm{\theta}}^{*}))=0$ holds for all $i\in[n]$. In addition, from the assumption that every critical point of the loss function $\ell$ is a global minimum, it follows that $z_{i}=-y_{i}\tilde{f}(x_{i};\tilde{\bm{\theta}}^{*})$ achieves the global minimum of the loss function $\ell$ and this further indicates that $\bm{\theta}^{*}$ is a global minimum of the empirical loss $\tilde{L}_{n}(\tilde{\bm{\theta}})$. Furthermore, since at every local minimum, all neurons are inactive, then two networks $\tilde{f}(\cdot;\tilde{\bm{\theta}}^{*})$ and $f(\cdot;\bm{\theta}^{*})$ are equivalent, i.e., $\tilde{f}(\cdot;\tilde{\bm{\theta}}^{*})=f(\cdot;\bm{\theta}^{*})$ holds for all $x\in\mathbb{R}^{d}$. Therefore, $z_{i}=-y_{i}{f}(x_{i};{\bm{\theta}}^{*})$ also achieves the global minimum of the loss function $\ell$ and this further indicates that $\bm{\theta}^{*}$ is a global minimum of the empirical loss $L_{n}(\bm{\theta})$. 
Finally, since every critical point of the loss function $\ell(z)$ satisfies $z<0$, then for every sample, $\ell'(-y_{i}f(x_{i};\bm{\theta}^{*}))=0$ indicates that $y_{i}f(x_{i};\bm{\theta}^{*})>0$, or, equivalently, $y_{i}=\sgn(f(x_{i};\bm{\theta}^{*}))$. Therefore, the set of parameters $\bm{\theta}^{*}$ also minimizes the training error. In summary, the set of parameters $\tilde{\bm{\theta}}^{*}=(\bm{\theta}^{*},a^{*}, \bm{w}^{*}, b^{*})$ minimizes the  loss function $\tilde{L}_{n}(\tilde{\bm{\theta}})$ and the set of parameters $\bm{\theta}^{*}$ simultaneously minimizes the empirical loss function ${L}_{n}(\bm{\theta})$ and the training error ${R}_{n}(\bm{\theta};f)$.
\end{proof}

\newpage
\section{Proof of Proposition~\ref{thm::single-monomial}}

\subsection{Important Lemmas}
Similar to the proof of Proposition~\ref{thm::single-monomial}, we need to first prove the following lemma. 
\begin{lemma}	\label{lemma3::a=0}
Under Assumption~\ref{assump::loss} and $\lambda>0$, if $\tilde{\bm{\theta}}^{*} = (\bm{\theta}^{*},a^{*},\bm{w}^{*},b^{*})$ is a local minimum of $\tilde{L}_{n}$,  then (i) $a^{*}=0$, (ii)  the following equation holds for all unit vector $(\bm{u}, v):\|\bm{u}\|^{2}_{2}+v^{2}=1$,
\begin{equation}\label{lemma3::eq-1}
\sum_{i=1}^n \ell'\left( -y_if(x_i;\bm{\theta}^{*})\right) y_i (\bm{u}^\top x_{i}+v)^{p} =0.
\end{equation}
\end{lemma}

\begin{proof}
\textbf{Proof of Lemma~\ref{lemma3::a=0} ($i$)}. To prove $a^{*}=0$, we only need to check the first order conditions of local minima. By assumption that $\tilde{\bm{\theta}}^{*}=(\bm{\theta}^{*}, a^{*},\bm{w}^{*},b^{*})$ is a local minimum of $\tilde{L}_{n}$, then the derivative of $\tilde{L}_{n}$ with respect to $a$ and $b$ at the point $\tilde{\bm{\theta}}^{*}$ are all zeros, i.e., 
	\begin{align}
	\left.\nabla_ a \tilde{L}_{n}(\tilde{\bm{\theta}})\right|_{\tilde{\bm{\theta}}=\tilde{\bm{\theta}}^{*}}  &=-\sum_{i=1}^n \ell'\left( -y_if(x_i;\bm{\theta}^{*})- y_i a^{*} ({\bm{w}^{*}}^\top x_i +b^{*})^{p}\right) y_i  ({\bm{w}^{*}}^\top x_i +b^{*})^{p} +\lambda a^{*}=0, \notag\\
	\left.\nabla_{(\bm{w}, b)} \tilde{L}_{n}(\tilde{\bm{ \theta}})\right|_{\tilde{\bm{\theta}}=\tilde{\bm{\theta}}^{*}} &=- pa^{*} \sum_{i=1}^n \ell'\left( -y_if(x_i;\bm{\theta}^{*})- y_ia^{*} ({\bm{w}^{*}}^\top x_i +b^{*})^{p}\right) y_i  ({\bm{w}^{*}}^\top x_i+b^{*})^{p-1}{{x_{i}}\choose{1}}=\bm{0}_{d+1}. \notag
	\end{align}
Taking the inner product of the both sides of the second equation with the vector $\bm{w}^{*}\choose {b^{*}}$, we have 
$$- pa^{*} \sum_{i=1}^n \ell'\left( -y_if(x_i;\bm{\theta}^{*})- y_ia^{*} ({\bm{w}^{*}}^\top x_i +b^{*})^{p}\right) y_i  ({\bm{w}^{*}}^\top x_i+b^{*})^{p}=0.$$
From above three equations, it is not difficult to see that $a^{*}$ satisfies $p\lambda {a^{*}}^{2}=0$ or, equivalently, $a^{*}=0$.  

\textbf{Proof of Lemma~\ref{lemma3::a=0} ($ii$)}. The main idea of the proof is to use the high order information of the local minimum to prove the Lemma. Due to the assumption that $\tilde{\bm{\theta}}=(\bm{\theta}^{*}, a^{*}, \bm{w}^{*}, b^{*})$ is a local minimum of the empirical loss function $\tilde{L}_{n}$, there exists a bounded local region such that the parameters $\tilde{\bm{\theta}}^{*}$ achieve the minimum loss value in this region, i.e.,  $\exists\delta\in(0,1)$ such that $\tilde{L}_{n}(\tilde{\bm{\theta}}^{*}+\bm{\Delta})\ge\tilde{L}_{n}(\tilde{\bm{\theta}}^{*})$ for $\forall\bm{\Delta}: \|\bm{\Delta}\|_{2}\le\delta$. 

Now, we use $\delta_{a}$, $\bm{\delta_{w}}$  and $\delta_{b}$ to denote the perturbations on the parameters $a$, $\bm{w}$ and $b$ respectively. Next, we consider the loss value at the point $\tilde{\bm{\theta}}^{*}+\bm{\Delta}=(\bm{\theta}^{*}, a^{*}+\delta_{a}, \bm{w}^{*}+\bm{\delta_{w}}, b^{*}+\delta_{b})$, where we set $|\delta_{a}|=e^{-1/\e}$, $\bm{\delta_{w}}=\e\bm{u}$ and $\delta_{b}=\e v$ for an arbitrary unit vector $(\bm{u}, v):\|\bm{u}\|^{2}_{2}+v^{2}=1$. Therefore, as $\e$ goes to zero, the perturbation magnitude $\|\bm{\Delta}\|_{2}$ also goes to zero and this indicates that there exists an $\e_{0}\in(0,1)$ such that  $\tilde{L}_{n}(\tilde{\bm{\theta}}^{*}+\bm{\Delta})\ge\tilde{L}_{n}(\tilde{\bm{\theta}}^{*})$ for $\forall \e\in[0, \e_{0})$. By $a^{*}=0$, the output of the model $\tilde{f}$ under parameters $\tilde{\bm{\theta}}^{*}+\bm{\Delta}$ can be expressed by 
$$\tilde{f}(x;\tilde{\bm{\theta}}^{*}+\bm{\Delta})=f(x;\bm{\theta}^{*})+\delta_{a}({\bm{w}^{*}}^{\top}x+b^{*}+\bm{\delta_{w}}^{\top}x+\delta_{b})^{p}.$$
        Let $g(x_{i};\bm{w}^{*}, \bm{\delta_{w}}^{}, b^{*},\delta_{b})=({\bm{w}^{*}}^{\top}x+b^{*}+\bm{\delta_{w}}^{\top}x+\delta_{b})^{p}$.
        For each sample $(x_{i}, y_{i})$ in the dataset, by the second order Taylor expansion with Lagrangian remainder, there exists a scalar $\xi_{i}\in[-|\delta_{a}|, |\delta_{a}|]$ depending on $\delta_{a}$ and $g(x_{i};\bm{w}^{*}, \bm{\delta_{w}}, b^{*})$ such that the following equation holds,
        \begin{align*}
        \ell(-y_{i}f(x_{i};\bm{\theta}^{*})-y_{i}&\delta_{a}g(x_{i};\bm{w}^{*}, \bm{\delta_{w}}^{}, b^{*},\delta_{b}))\\
        &=\ell(-y_{i}f(x_{i};\bm{\theta}^{*}))+\ell'(-y_{i}f(x_{i};\bm{\theta}^{*}))(-y_{i})\delta_{a}g(x_{i};\bm{w}^{*}, \bm{\delta_{w}}^{}, b^{*},\delta_{b})\\
        &\quad+\frac{1}{2!}\ell''(-y_{i}f(x_{i};\bm{\theta}^{*})-y_{i}\xi_{i}g(x_{i};\bm{w}^{*}, \bm{\delta_{w}}^{}, b^{*},\delta_{b}))\delta^{2}_{a}g^{2}(x_{i};\bm{w}^{*}, \bm{\delta_{w}}^{}, b^{*},\delta_{b}).
        \end{align*}
         Clearly, for all $\varepsilon<1$,  $|\delta_{a}|<e^{-1}$ and $\|\bm{\delta_{{w}}}\|^{2}_{2}+|\delta_{b}|^{2}<1$, we have 
        $$g(x_{i};\bm{w}^{*}, \bm{\delta_{w}}^{}, b^{*},\delta_{b})=({\bm{w}^{*}}^{\top}x_{i}+b^{*}+\bm{\delta_{w}}^{\top}x_{i}+\delta_{b})^{p}\le (|{\bm{w}^{*}}^{\top}x_{i}+b^{*}|+\|x_{i}\|_{2}^{2}+1)^{p}.$$
        Since $|\xi_{i}|<|\delta_{a}|<e^{-1}$, then for each $i\in[n]$, there exists a constant $C_{i}$ depend on $\bm{\theta}^{*}, \bm{w}^{*}, b^{*}$ such that 
        $$|\ell''(-y_{i}f(x_{i};\bm{\theta}^{*})-z)|< C_{i}$$
        holds for all $z\in[-(|{\bm{w}^{*}}^{\top}x_{i}+b^{*}|+\|x_{i}\|_{2}^{2}+1)^{p}, (|{\bm{w}^{*}}^{\top}x_{i}+b^{*}|+\|x_{i}\|_{2}^{2}+1)^{p}]$.
        
        Since $\tilde{\bm{\theta}}^{*}$ is a local minimum, then there exists $\varepsilon_{0}\in(0, 1)$ such that the inequality
        \begin{align*}
        L_{n}(\tilde{\bm{\theta}}^{*}+\bm{\Delta}) - L_{n}(\tilde{\bm{\theta}}^{*})&=
        \sum_{i=1}^{n}\ell(-y_{i}f(x_{i};\bm{\theta}^{*})-y_{i}\delta_{a}g(x_{i};\bm{w}^{*}, \bm{\delta_{w}}^{}, b^{*},\delta_{b}))+\frac{\lambda\delta_{a}^{2}}{2}-\sum_{i=1}^{n}\ell(-y_{i}f(x_{i};\bm{\theta}^{*}))\\
        &=\sum_{i=1}^{n}\ell'(-y_{i}f(x_{i};\bm{\theta}^{*}))(-y_{i})\delta_{a}g(x_{i};\bm{w}^{*}, \bm{\delta_{w}}^{}, b^{*},\delta_{b})+\frac{\lambda\delta_{a}^{2}}{2}\\
        &\quad+\sum_{i=1}^{n}\frac{1}{2!}\ell''(-y_{i}f(x_{i};\bm{\theta}^{*})-y_{i}\xi_{i}g(x_{i};\bm{w}^{*}, \bm{\delta_{w}}^{}, b^{*},\delta_{b}))\delta^{2}_{a}g^{2}(x_{i};\bm{w}^{*}, \bm{\delta_{w}}^{}, b^{*},\delta_{b})\\
        &\ge 0
        \end{align*}
        holds for all $\varepsilon<\varepsilon_{0}$. In addition, we have 
        \begin{align*}
        \sum_{i=1}^{n}&\frac{1}{2!}\ell''\left(-y_{i}f(x_{i};\bm{\theta}^{*})-y_{i}\xi_{i}g(x_{i};\bm{w}^{*}, \bm{\delta_{w}}^{}, b^{*},\delta_{b})\right)\delta^{2}_{a}g^{2}(x_{i};\bm{w}^{*}, \bm{\delta_{w}}^{}, b^{*},\delta_{b})\\
        &\le \sum_{i=1}^{n}C_{i}\delta^{2}_{a}g^{2}(x_{i};\bm{w}^{*}, \bm{\delta_{w}}^{}, b^{*},\delta_{b})\\
        &\le \exp(-2/\varepsilon)\sum_{i=1}^{n}C_{i}(|{\bm{w}^{*}}^{\top}x_{i}+b^{*}|+\|x_{i}\|_{2}^{2}+1)^{p}
        \end{align*}
       Recall that scalar $C_{i}$ only depends on $\tilde{\bm{\theta}}^{*}=(\bm{\theta}^{*}, \bm{w}^{*}, b^{*})$ and $x_{i}$, thus the scalar $\sum_{i=1}^{n}C_{i}(|{\bm{w}^{*}}^{\top}x_{i}+b^{*}|+\|x_{i}\|_{2}^{2}+1)^{p}+\lambda/2$ can be viewed as a scalar depending only on parameters $\tilde{\bm{\theta}}^{*}$ and dataset $\mathcal{D}$. Thus, for any $\varepsilon:\varepsilon<\varepsilon_{0}$ and for any $\sgn(\delta_{a})\in\{-1, 1\}$, the inequality 
       \begin{equation}
       \sgn(\delta_{a})\exp(-1/\varepsilon)\sum_{i=1}^{n}\ell'(-y_{i}f(x_{i};\bm{\theta}^{*}))(-y_{i})g(x_{i};\bm{w}^{*}, \varepsilon\bm{u}, b^{*},\delta_{b})+C(\tilde{\bm{\theta}}^{*},\mathcal{D})\exp(-2/\varepsilon)\ge 0
       \end{equation}
       always holds.  
       This indicates that for any $\varepsilon:\varepsilon<\varepsilon_{0}$ and for any $\sgn(\delta_{a})\in\{-1, 1\}$, the inequality 
        \begin{equation}\label{lemma3::eq-2}
       \sgn(\delta_{a})\sum_{i=1}^{n}\ell'(-y_{i}f(x_{i};\bm{\theta}^{*}))(-y_{i})({\bm{w}^{*}}^{\top}x+b^{*}+\e\bm{u}^{\top}x+\e v)^{p}+C(\tilde{\bm{\theta}}^{*},\mathcal{D})\exp(-1/\varepsilon)\ge 0
       \end{equation}
       always holds. 
       Since 
       $$({\bm{w}^{*}}^{\top}x+b^{*}+\e\bm{u}^{\top}x+\e v)^{p}=\sum_{q=0}^{p}{p\choose q}\e^{q}({\bm{w}^{*}}^{\top}x+b^{*})^{p-q}(\bm{u}^{\top}x+v)^{q}$$
       
       We now proceed by induction. For the base case where $p=0$, 
       for each $\sgn(\delta_{a})\in\{-1, 1\}$, we take the limit on the both sides of inequality~\eqref{lemma3::eq-2} as $\varepsilon\rightarrow 0$ and thus obtain   
       \begin{equation}
       \sgn(\delta_{a})\sum_{i=1}^{n}\ell'(-y_{i}f(x_{i};\bm{\theta}^{*}))(-y_{i})({\bm{w}^{*}}^{\top}x_{i}+b^{*})^{p}\ge 0,
       \end{equation}
       which further establishes the base case
       \begin{equation}
      \sgn(\delta_{a})\sum_{i=1}^{n}\ell'(-y_{i}f(x_{i};\bm{\theta}^{*}))(-y_{i})({\bm{w}^{*}}^{\top}x_{i}+b^{*})^{p} = 0.
 	\end{equation}
	The inductive hypothesis is that the equality 
	\begin{equation}\label{lemma3::eq-3}
	\sum_{i=1}^{n}\ell'(-y_{i}f(x_{i};\bm{\theta}^{*}))(-y_{i})({\bm{w}^{*}}^{\top}x_{i}+b^{*})^{p-q}(\bm{u}^{\top}x_{i}+v)^{q} = 0
	\end{equation}
	 holds for all $q=0,...,k-1$. Now we need to prove that the equality~\eqref{lemma3::eq-3} holds for $q=k$. 
	 Since the equality holds for all $q= 0,..., k-1$, then we have 
	 \begin{align}
	 \sgn(\delta_{a})&\sum_{i=1}^{n}\ell'(-y_{i}f(x_{i};\bm{\theta}^{*}))(-y_{i})\cdot\\
	 &\cdot\left[\frac{\sum_{q=0}^{p}{p\choose q}\e^{q}({\bm{w}^{*}}^{\top}x+b^{*})^{p-q}(\bm{u}^{\top}x+v)^{q}-\sum_{q=0}^{k-1}{p\choose q}\e^{q}({\bm{w}^{*}}^{\top}x+b^{*})^{p-q}(\bm{u}^{\top}x+v)^{q}}{\varepsilon^{k}}\right]\notag\\
	 &+C(\tilde{\bm{\theta}}^{*},\mathcal{D})1/\varepsilon^{k}\exp(-1/\varepsilon)\ge 0\label{lemma3::eq-4}
	 \end{align}
	 Taking the limit on the both sides of Eq.~\ref{lemma::eq-4}, we obtain that the inequality
	 \begin{equation}\label{lemma3::eq-5}
	\sgn(\delta_{a})\sum_{i=1}^{n}\ell'(-y_{i}f(x_{i};\bm{\theta}^{*}))(-y_{i})({\bm{w}^{*}}^{\top}x_{i}+b^{*})^{p-k}(\bm{u}^{\top}x_{i}+v)^{k} \ge 0,
	\end{equation}
	holds for every $\sgn(\delta_{a})\in\{-1, 1\}$ and this further implies 
	\begin{equation}\label{lemma3::eq-6}
	\sum_{i=1}^{n}\ell'(-y_{i}f(x_{i};\bm{\theta}^{*}))(-y_{i})({\bm{w}^{*}}^{\top}x_{i}+b^{*})^{p-k}(\bm{u}^{\top}x_{i}+v)^{k}= 0.
	\end{equation}
	Thus Eq.~\eqref{lemma3::eq-6} finishes our induction. 
\end{proof}

\subsection{Proof  of Proposition~\ref{thm::single-monomial}}
\begin{proof}
For every dataset $\mathcal{D}$ satisfying Assumption~\ref{assump::realizability}, by the Lagrangian interpolating polynomial, there always exists a polynomial  $P(x)=\sum_{j}c_{j}\pi_{j}(x)$ defined on $\mathbb{R}^{d}$ such that it can correctly classify all samples in the dataset with margin at least one, i.e., $y_{i}P(x_{i})\ge 1,\forall i\in[n]$, where $\pi_{j}$ denotes the $j$-th monomial in the polynomial $P(x)$. 
Therefore, from Lemma~\ref{lemma3::a=0} and \ref{lemma::tensor}, it follows that 
\begin{align*}
\sum_{i=1}^{n}\ell'(-y_{i}f(x_{i};\bm{\theta}^{*}))y_{i}P(x_{i})=\sum_{j}c_{j}\sum_{i=1}^{n}\ell'(-y_{i}f(x_{i};\bm{\theta}^{*}))y_{i}\pi_{j}(x_{i})=0.
\end{align*}
Since $y_{i}P(x_{i})\ge 1$ and $e^{{\bm{w}^{*}}^{\top}x_{i}+b^{*}}>0$ hold for  $\forall i\in[n]$ and the loss function $\ell$ is a non-decreasing function, i.e., $\ell'(z)\ge 0,\forall z\in\mathbb{R}$, then $\ell'(-y_{i}f(x_{i};\bm{\theta}^{*}))=0$ holds for all $i\in[n]$. In addition, from the assumption that every critical point of the loss function $\ell$ is a global minimum, it follows that $z_{i}=-y_{i}f(x_{i};\bm{\theta}^{*})$ achieves the global minimum of the loss function $\ell$ and this further indicates that $\bm{\theta}^{*}$ is a global minimum of the empirical loss $L_{n}(\bm{\theta})$. Furthermore, since at every local minimum, the exponential neuron is inactive, $a^{*}=0$, then the set of parameters $\tilde{\bm{\theta}}^{*}$ is a global minimum of the loss function $\tilde{L}(\tilde{\bm{\theta}})$. Finally, since every critical point of the loss function $\ell(z)$ satisfies $z<0$, then for every sample, $\ell'(-y_{i}f(x_{i};\bm{\theta}^{*}))=0$ indicates that $y_{i}f(x_{i};\bm{\theta}^{*})>0$, or, equivalently, $y_{i}=\sgn(f(x_{i};\bm{\theta}^{*}))$. Therefore, the set of parameters $\bm{\theta}^{*}$ also minimizes the training error. In summary, the set of parameters $\tilde{\bm{\theta}}^{*}=(\bm{\theta}^{*},a^{*}, \bm{w}^{*}, b^{*})$ minimizes the  loss function $\tilde{L}_{n}(\tilde{\bm{\theta}})$ and the set of parameters $\bm{\theta}^{*}$ simultaneously minimizes the empirical loss function $L_{n}(\bm{\theta})$ and the training error $R_{n}(\bm{\theta};f)$. 
\end{proof}

\newpage

\section{Proof of Proposition~\ref{thm::convex}}
\begin{proof}
The proof Proposition~\ref{thm::convex} is based on Lemma~\ref{lemma::a=0} and \ref{lemma::tensor}.

We consider the following three cases: (1) the case where there exists a set of parameters $\bm{\theta}$ such that the neural network $f(x_{i};\bm{\theta})$ can correctly classify all point in the dataset; (2) the case where all points in the dataset can be correctly classified by a certain polynomial but cannot be correctly classified by the network $f(\cdot;\bm{\theta})$ for every set of parameters $\bm{\theta}$; (3) the case where all points in the dataset cannot be separable by a certain polynomial. 

\textbf{Case (1):} The proof for the case (1) is exactly the same as the proof of Theorem~\ref{thm::single-exp}. 

\textbf{Case (2):} Similar to the proof of case (1), we can prove that $\ell'_{i}(-y_{i}f(x_{i};\bm{\theta}^{*}))=0$ holds for all $i\in[n]$. However, since there does not exist a set of parameters $\bm{\theta}$ such that $y_{i}f(x_{i};\bm{\theta}^{})>0$, then this leads to the contradiction with our assumption that $\tilde{\bm{\theta}}^{*}$ is a local minimum of the empirical loss $\tilde{L}_{n}(\tilde{\bm{\theta}})$. This means that in this case, the empirical loss function $\tilde{L}_{n}$ does not have any local minimum. In addition, since the empirical loss $\tilde{L}_{n}(\tilde{\bm{\theta}})$ is strongly convex with respect to $a$, then every critical point of $\tilde{L}_{n}$ cannot be a local maximum. Therefore, every critical point of the empirical loss function in this case is a saddle point. 

\textbf{Case (3):} This case is the most complicated case, since there does not exist a polynomial such that this polynomial can correctly classify all points in the dataset. This indicates that there exists two samples in the dataset with the same feature vectors but with different labels, i.e., $\exists (x_{i},y_{i}), (x_{j},y_{j})\in\mathcal{D}: x_{i}=x_{j}, y_{i}\neq y_{j}$. Now we split the whole dataset into $K\le n$ mutually exclusive dataset $\mathcal{D}_{1},...,\mathcal{D}_{K}$ such that these datasets satisfy: (1) the union of these datasets is the dataset $\mathcal{D}$, i.e., $\cup_{i=1}^{K}\mathcal{D}_{i}=\mathcal{D}$; (2) for every dataset $\mathcal{D}_{i}$,  every  pair of samples $(x_{j}, y_{j}), (x_{k},y_{k})$ in this dataset has the same feature vector; (3) any pair of  samples $(x_{i}, y_{i}), (x_{j},y_{j})$ from two different datasets has different feature vector. 

By Lagrangian interpolation polynomial, for each dataset $\mathcal{D}_{k}$, there always exists a polynomial $P_{k}$ such that $P_{k}(x)=1$ for all $(x, y)\in\mathcal{D}_{k}$ and $P_{k}(x)=0$ for all $(x,y)\notin\mathcal{D}_{k}, (x,y)\in\mathcal{D}$. Therefore, by Lemma~\ref{lemma::a=0} and \ref{lemma::tensor}, we have 
\begin{align*}
0=\sum_{i=1}^{n}\ell'(-y_{i}f(x_{i};\bm{\theta}^{*}))e^{{\bm{w}^{*}}^{\top}x_{i}+b^{*}}y_{i}P_{k}(x_{i})&=\sum_{i:(x_{i},y_{i})\in\mathcal{D}_{k}}\ell'(-y_{i}f(x_{i};\bm{\theta}^{*}))e^{{\bm{w}^{*}}^{\top}x_{i}+b^{*}}y_{i}\\
&=e^{{\bm{w}^{*}}^{\top}z_{k}+b^{*}}\sum_{i:(x_{i},y_{i})\in\mathcal{D}_{k}}\ell'(-y_{i}f(x_{i};\bm{\theta}^{*}))y_{i},
\end{align*}
where we use $z_{k}$ to denote the feature vector of all samples in the dataset $\mathcal{D}_{k}$ and the last equality follows from the property that all samples in the same dataset have the same feature vector. Therefore, we obtain that for each dataset $\mathcal{D}_{k}$, the following equality holds 
\begin{equation}\label{eq::thm2-1}
\sum_{i:(x_{i},y_{i})\in\mathcal{D}_{k}}\ell'(-y_{i}f(x_{i};\bm{\theta}^{*}))y_{i}=0.
\end{equation}
Thus,
$$\ell'(-f(z_{k};\bm{\theta}^{*}))\sum_{i:(x_{i},y_{i})\in\mathcal{D}_{k}}\mathbb{I}\{y_{i}=1\}=\ell'(f(z_{k};\bm{\theta}^{*}))\sum_{i:(x_{i},y_{i})\in\mathcal{D}_{k}}\mathbb{I}\{y_{i}=1\},$$
where we use $z_{k}$ to denote the feature vector of all samples in the dataset $\mathcal{D}_{k}$.
Therefore, it is easy to see that if $\sum_{i:(x_{i},y_{i})\in\mathcal{D}_{k}}\mathbb{I}\{y_{i}=1\}>\sum_{i:(x_{i},y_{i})\in\mathcal{D}_{k}}\mathbb{I}\{y_{i}=1\}$, then $$\ell'(-f(z_{k};\bm{\theta}^{*}))<\ell'(f(z_{k};\bm{\theta}^{*})),$$
since $\ell'(-f(z_{k};\bm{\theta}^{*}))$ and $\ell'(f(z_{k};\bm{\theta}^{*}))$ cannot be zero at the same time. In addition,  due to the assumption that $\ell$ is convex, then $\ell'(z)$ is an increasing function and this indicates that 
$$f(z_{k};\bm{\theta}^{*})>0,$$
or equivalently, the prediction from the neural network $f(\cdot;\bm{\theta}^{*})$ on the samples in the dataset $\mathcal{D}_{k}$ is the label of the majority samples in the dataset $\mathcal{D}_{k}$. This indicates that the network $f(\cdot;\bm{\theta})$ achieves the minimum misclassification rate on the dataset $\mathcal{D}_{k}$. 
Similarly, using the same analysis, we can prove that the same result holds for the case $\sum_{i:(x_{i},y_{i})\in\mathcal{D}_{k}}\mathbb{I}\{y_{i}=1\}<\sum_{i:(x_{i},y_{i})\in\mathcal{D}_{k}}\mathbb{I}\{y_{i}=1\}$ and the case $\sum_{i:(x_{i},y_{i})\in\mathcal{D}_{k}}\mathbb{I}\{y_{i}=1\}=\sum_{i:(x_{i},y_{i})\in\mathcal{D}_{k}}\mathbb{I}\{y_{i}=1\}$. Therefore, the network $f(\cdot;\bm{\theta}^{*})$ achieves the minimum misclassification rate on the dataset $\mathcal{D}$. 

Next, we will prove that $\tilde{\bm{\theta}}^{*}$ is also the global minimum of the loss function $\tilde{L}_{n}(\tilde{\bm{\theta}})$. For any set of parameters $\hat{\bm{\theta}}=(\hat{\bm{\theta}},\hat{a},\hat{\bm{w}},\hat{b})$, we only need to show that 
$$\tilde{L}_{n}(\hat{\bm{\theta}})\ge \tilde{L}_{n}(\tilde{\bm{\theta}}^{*}).$$
Since 
\begin{align*}
\tilde{L}_{n}(\hat{\bm{\theta}})=\sum_{i=1}^{n}\ell(-y_{i}\tilde{f}(x_{i};\hat{\bm{\theta}}))+\frac{\lambda \hat{a}^{2}}{2}=\sum_{k=1}^{K}\sum_{i:(x_{i},y_{i})\in\mathcal{D}_{k}}\ell(-y_{i}\tilde{f}(x_{i};\hat{\bm{\theta}}))+\frac{\lambda \hat{a}^{2}}{2},
\end{align*}
and 
\begin{align*}
\sum_{i:(x_{i},y_{i})\in\mathcal{D}_{k}}\ell\left(-y_{i}\tilde{f}(x_{i};\hat{\bm{\theta}})\right)&\ge\sum_{i:(x_{i},y_{i})\in\mathcal{D}_{k}}\ell\left(-y_{i}\tilde{f}(x_{i};\tilde{\bm{\theta}}^{*})\right)\\
&\quad+\sum_{i:(x_{i},y_{i})\in\mathcal{D}_{k}}\ell'\left(-y_{i}\tilde{f}(x_{i};\tilde{\bm{\theta}}^{*})\right)(-y_{i})(\tilde{f}(x_{i};\hat{\bm{\theta}})-\tilde{f}(x_{i};\tilde{\bm{\theta}}^{*}))
\end{align*}
and the fact that all samples in the dataset $D_{k}$ has the same feature vector $z_{k}$, then 
\begin{align*}
\sum_{i:(x_{i},y_{i})\in\mathcal{D}_{k}}\ell\left(-y_{i}\tilde{f}(x_{i};\hat{\bm{\theta}})\right)&\ge\sum_{i:(x_{i},y_{i})\in\mathcal{D}_{k}}\ell\left(-y_{i}\tilde{f}(x_{i};\tilde{\bm{\theta}}^{*})\right)\\
&\quad+(\tilde{f}(z_{k};\hat{\bm{\theta}})-\tilde{f}(z_{k};\tilde{\bm{\theta}}^{*}))\sum_{i:(x_{i},y_{i})\in\mathcal{D}_{k}}\ell'\left(-y_{i}\tilde{f}(x_{i};\tilde{\bm{\theta}}^{*})\right)(-y_{i})\\
&=\sum_{i:(x_{i},y_{i})\in\mathcal{D}_{k}}\ell\left(-y_{i}\tilde{f}(x_{i};\tilde{\bm{\theta}}^{*})\right)
\end{align*}
where the equality holds from Equation~\eqref{eq::thm2-1} and the fact that $\tilde{f}$ and $f$ are equivalent at every local minimum $\tilde{\bm{\theta}}^{*}$. In addition, since $\hat{a}^{2}\ge0$, then 
\begin{align*}
\tilde{L}_{n}(\hat{\bm{\theta}})&=\sum_{k=1}^{K}\sum_{i:(x_{i},y_{i})\in\mathcal{D}_{k}}\ell(-y_{i}\tilde{f}(x_{i};\hat{\bm{\theta}}))+\frac{\lambda \hat{a}^{2}}{2}\\
&\ge \sum_{k=1}^{K}\sum_{i:(x_{i},y_{i})\in\mathcal{D}_{k}}\ell\left(-y_{i}\tilde{f}(x_{i};\tilde{\bm{\theta}}^{*})\right)=\tilde{L}_{n}(\tilde{\bm{\theta}}^{*}).
\end{align*}
Therefore, $\tilde{\bm{\theta}}^{*}$ is also a global minimum of the loss function $\tilde{L}_{n}$. Using the same analysis, we can prove that $\bm{\theta}^{*}$ is also a global minimum of the loss function $L_{n}$.
\end{proof}

\newpage

\section{Proof of Proposition~\ref{thm::stationary}}

\subsection{Important Lemma}
\begin{lemma}	\label{lemma4::a=0}
Under Assumption~\ref{assump::loss} and $\lambda>0$, if $\tilde{\bm{\theta}}^{*} = (\bm{\theta}^{*},a^{*},\bm{w}^{*},b^{*})$ is a $k$-th order stationary point of $\tilde{L}_{n}$,  then (i) $a^{*}=0$, (ii) for integer $q\in[0, \lfloor k/2\rfloor]$, the following equation holds for all unit vector $\bm{u}:\|\bm{u}\|_{2}=1$,
\begin{equation}\label{lemma4::eq-1}
\sum_{i=1}^n \ell'\left( -y_if(x_i;\bm{\theta}^{*})\right) y_i e^{{\bm{w}^{*}}^\top x_i +b^{*}} (\bm{u}^\top x_{i})^{q} =0.
\end{equation}
\end{lemma}

\begin{proof}
\textbf{Proof of Lemma~\ref{lemma4::a=0} ($i$)}. To prove $a^{*}=0$, we only need to check the first order conditions of the $k$-th order stationary point. By assumption that $\tilde{\bm{\theta}}^{*}=(\bm{\theta}^{*}, a^{*},\bm{w}^{*},b^{*})$ is a $k$-th order stationary point of $\tilde{L}_{n}$, then the derivative of $\tilde{L}_{n}$ with respect to $a$ and $b$ at the point $\tilde{\bm{\theta}}^{*}$ are all zeros, i.e., 
	\begin{align}
	\left.\nabla_ a \tilde{L}_{n}(\tilde{\bm{\theta}})\right|_{\tilde{\bm{\theta}}=\tilde{\bm{\theta}}^{*}}  &=-\sum_{i=1}^n \ell'\left( -y_if(x_i;\bm{\theta}^{*})- y_ia e^{{\bm{w}^{*}}^\top x_i +b^{*}}\right) y_i  \exp({\bm{w}^{*}}^\top x_i +b^{*}) +\lambda a^{*}=0, \notag\\
	\left.\nabla_b \tilde{L}_{n}(\tilde{\bm{ \theta}})\right|_{\tilde{\bm{\theta}}=\tilde{\bm{\theta}}^{*}} &=- a^{*} \sum_{i=1}^n \ell'\left( -y_if(x_i;\bm{\theta}^{*})- y_ia e^{{\bm{w}^{*}}^\top x_i +b^{*}}\right) y_i  \exp({\bm{w}^{*}}^\top x_i+b^{*})=0. \notag
	\end{align}
From above two equations, it is not difficult to see that $a^{*}$ satisfies $\lambda {a^{*}}^{2}=0$ or, equivalently, $a^{*}=0$.  

\textbf{Proof of Lemma~\ref{lemma4::a=0} ($ii$)}. The main idea of the proof is to use the high order information of the $k$-th order stationary point to prove the Lemma. Due to the assumption that $\tilde{\bm{\theta}}=(\bm{\theta}^{*}, a^{*}, \bm{w}^{*}, b^{*})$ is a $k$-th order stationary point of the empirical loss function $\tilde{L}_{n}$, there exists a positive constant $C$ and $\exists\delta\in(0,1)$ such that $\tilde{L}_{n}(\tilde{\bm{\theta}}^{*}+\bm{\Delta})\ge\tilde{L}_{n}(\tilde{\bm{\theta}}^{*})-C\|\bm{\Delta}\|_{2}^{k+1}$ for $\forall\bm{\Delta}: \|\bm{\Delta}\|_{2}\le\delta$. 

Now, we use $\delta_{a}$, $\bm{\delta_{w}}$ to denote the perturbations on the parameters $a$ and $\bm{w}$, respectively. Next, we consider the loss value at the point $\tilde{\bm{\theta}}^{*}+\bm{\Delta}=(\bm{\theta}^{*}, a^{*}+\delta_{a}, \bm{w}^{*}+\bm{\delta_{w}}, b^{*})$, where we set $|\delta_{a}|=\e^{(k+1)/2}$ and $\bm{\delta_{w}}=\e\bm{u}$ for an arbitrary unit vector $\bm{u}:\|\bm{u}\|_{2}=1$. Therefore, as $\e$ goes to zero, the perturbation magnitude $\|\bm{\Delta}\|_{2}$ also goes to zero and this indicates that there exists an $\e_{0}\in(0,1)$ such that  $\tilde{L}_{n}(\tilde{\bm{\theta}}^{*}+\bm{\Delta})\ge\tilde{L}_{n}(\tilde{\bm{\theta}}^{*})-C(\e^{2}+\e^{k+1})^{(k+1)/2}$ for $\forall \e\in[0, \e_{0})$. By $a^{*}=0$, the output of the model $\tilde{f}$ under parameters $\tilde{\bm{\theta}}^{*}+\bm{\Delta}$ can be expressed by 
$$\tilde{f}(x;\tilde{\bm{\theta}}^{*}+\bm{\Delta})=f(x;\bm{\theta}^{*})+\delta_{a}\exp(\bm{\delta_{w}}^{\top}x)\exp({\bm{w}^{*}}^{\top}x+b^{*}).$$
        Let $g(x_{i};\bm{w}^{*}, \bm{\delta_{w}}^{}, b^{*})=\exp({\bm{\delta_{w}}^{}}^{\top}x_{i})\exp({\bm{w}^{*}}^{\top}x_{i}+b^{*})$.
        For each sample $(x_{i}, y_{i})$ in the dataset, by the second order Taylor expansion with Lagrangian remainder, there exists a scalar $\xi_{i}\in[-|\delta_{a}|, |\delta_{a}|]$ depending on $\delta_{a}$ and $g(x_{i};\bm{w}^{*}, \bm{\delta_{w}}, b^{*})$ such that the following equation holds,
        \begin{align*}
        \ell(-y_{i}f(x_{i};\bm{\theta}^{*})-y_{i}&\delta_{a}g(x_{i};\bm{w}^{*}, \bm{\delta_{w}}, b^{*}))\\
        &=\ell(-y_{i}f(x_{i};\bm{\theta}^{*}))+\ell'(-y_{i}f(x_{i};\bm{\theta}^{*}))(-y_{i})\delta_{a}g(x_{i};\bm{w}^{*}, \bm{\delta_{w}}, b^{*})\\
        &\quad+\frac{1}{2!}\ell''(-y_{i}f(x_{i};\bm{\theta}^{*})-y_{i}\xi_{i}g(x_{i};\bm{w}^{*},\bm{\delta_{w}}, b^{*}))\delta^{2}_{a}g^{2}(x_{i};\bm{w}^{*}, \bm{\delta_{w}}, b^{*}).
        \end{align*}
        Clearly, for all $\varepsilon<1$,  $|\delta_{a}|<1$ and $\|\bm{\delta_{{w}}}\|_{2}<1$, we have 
        $$g(x_{i};\bm{w}^{*}, \bm{\delta_{w}}, b^{*})=\exp(\bm{\delta_{w}}^{\top}x_{i})\exp({\bm{w}^{*}}^{\top}x_{i}+b^{*})\le \exp(\|x_{i}\|_{2})\exp({\bm{w}^{*}}^{\top}x_{i}+b^{*}).$$
        Since $|\xi_{i}|<|\delta_{a}|<1$, then for each $i\in[n]$, there exists a constant $C_{i}$ depend on $\bm{\theta}^{*}, \bm{w}^{*}, b^{*}$ such that 
        $$|\ell''(-y_{i}f(x_{i};\bm{\theta}^{*})-z)|< C_{i}$$
        holds for all $z\in[-\exp(\|x_{i}\|_{2}+{\bm{w}^{*}}^{\top}x_{i}+b^{*}), \exp(\|x_{i}\|_{2}+{\bm{w}^{*}}^{\top}x_{i}+b^{*})]$.
        
        Since $\tilde{\bm{\theta}}^{*}$ is a $k$-th order stationary point, then there exists $\varepsilon_{0}\in(0, 1)$ such that the inequality
        \begin{align*}
        L_{n}(\tilde{\bm{\theta}}^{*}+\bm{\Delta}) - L_{n}(\tilde{\bm{\theta}}^{*})&=
        \sum_{i=1}^{n}\ell(-y_{i}f(x_{i};\bm{\theta}^{*})-y_{i}\delta_{a}g(x_{i};\bm{w}^{*}, \bm{\delta_{w}}, b^{*}))+\frac{\lambda\delta_{a}^{2}}{2}-\sum_{i=1}^{n}\ell(-y_{i}f(x_{i};\bm{\theta}^{*}))\\
        &=\sum_{i=1}^{n}\ell'(-y_{i}f(x_{i};\bm{\theta}^{*}))(-y_{i})\delta_{a}g(x_{i};\bm{w}^{*}, \bm{\delta_{w}}, b^{*})+\frac{\lambda\delta_{a}^{2}}{2}\\
        &\quad+\sum_{i=1}^{n}\frac{1}{2!}\ell''(-y_{i}f(x_{i};\bm{\theta}^{*})-y_{i}\xi_{i}g(x_{i};\bm{w}^{*},\bm{\delta_{w}}, b^{*}))\delta^{2}_{a}g^{2}(x_{i};\bm{w}^{*}, \bm{\delta_{w}}, b^{*})\\
        &\ge -C(\e^{2}+\e^{k+1})^{(k+1)/2}\ge -C'\e^{k+1}
        \end{align*}
        holds for all $\varepsilon<\min\{\varepsilon_{0},1\}$ and some positive constant $C'$. In addition, we have 
        \begin{align*}
        \sum_{i=1}^{n}\ell''(-y_{i}f(x_{i};\bm{\theta}^{*})-&y_{i}\xi_{i}g(x_{i};\bm{w}^{*},\bm{\delta_{w}}, b^{*}))\delta^{2}_{a}g^{2}(x_{i};\bm{w}^{*}, \bm{\delta_{w}}, b^{*})\\
        &\le \sum_{i=1}^{n}C_{i}\delta^{2}_{a}g^{2}(x_{i};\bm{w}^{*}, \bm{\delta_{w}}, b^{*})\\
        &\le \e^{k+1}\sum_{i=1}^{n}C_{i}\exp(-\|x_{i}\|_{2}+{\bm{w}^{*}}^{\top}x_{i}+b^{*})
        \end{align*}
       Recall that scalar $C_{i}$ only depends on $\tilde{\bm{\theta}}^{*}=(\bm{\theta}^{*}, \bm{w}^{*}, b^{*})$ and $x_{i}$, thus the scalar $C(\tilde{\bm{\theta}}^{*},\mathcal{D})=\sum_{i=1}^{n}C_{i}\exp(-\|x_{i}\|_{2}+{\bm{w}^{*}}^{\top}x_{i}+b^{*})+\lambda/2+C'$ can be viewed as a scalar depending only on parameters $\tilde{\bm{\theta}}^{*}$ and dataset $\mathcal{D}$. Thus, for any $\varepsilon:\varepsilon<\varepsilon_{0}$ and for any $\sgn(\delta_{a})\in\{-1, 1\}$, the inequality 
       \begin{equation}
       \sgn(\delta_{a})\e^{(k+1)/2}\sum_{i=1}^{n}\ell'(-y_{i}f(x_{i};\bm{\theta}^{*}))(-y_{i})g(x_{i};\bm{w}^{*}, \varepsilon\bm{u}, b^{*})+C(\tilde{\bm{\theta}}^{*},\mathcal{D})\e^{k+1}\ge 0
       \end{equation}
       always holds.  
       This indicates that for any $\varepsilon:\varepsilon<\varepsilon_{0}$ and for any $\sgn(\delta_{a})\in\{-1, 1\}$, the inequality 
        \begin{equation}\label{lemma4::eq-2}
       \sgn(\delta_{a})\sum_{i=1}^{n}\ell'(-y_{i}f(x_{i};\bm{\theta}^{*}))(-y_{i})\exp(\varepsilon\bm{u}^{\top}x_{i})\exp({\bm{w}^{*}}^{\top}x_{i}+b^{*})+C(\tilde{\bm{\theta}}^{*},\mathcal{D})\e^{(k+1)/2}\ge 0
       \end{equation}
       always holds. 
       We now proceed by induction. For the base case where $p=0$, 
       for each $\sgn(\delta_{a})\in\{-1, 1\}$, we take the limit on the both sides of inequality~\eqref{lemma4::eq-2} as $\varepsilon\rightarrow 0$ and thus obtain   
       \begin{equation}
       \sgn(\delta_{a})\sum_{i=1}^{n}\ell'(-y_{i}f(x_{i};\bm{\theta}^{*}))(-y_{i})\exp({\bm{w}^{*}}^{\top}x_{i}+b^{*})\ge 0,
       \end{equation}
       which further establishes the base case
       \begin{equation}
       \sum_{i=1}^{n}\ell'(-y_{i}f(x_{i};\bm{\theta}^{*}))(-y_{i})\exp({\bm{w}^{*}}^{\top}x_{i}+b^{*}) = 0.
 	\end{equation}
	The inductive hypothesis is that the equality 
	\begin{equation}\label{lemma4::eq-3}
	\sum_{i=1}^{n}\ell'(-y_{i}f(x_{i};\bm{\theta}^{*}))(-y_{i})\exp({\bm{w}^{*}}^{\top}x_{i}+b^{*})(\bm{u}^{\top}x_{i})^{j} = 0
	\end{equation}
	 holds for all $j=0,...,s-1$ and $s\le\lfloor k/2\rfloor$. Now we need to prove that the equality~\eqref{lemma4::eq-3} holds for $j=s$. 
	 Since the equality holds for all $j= 0,..., s-1$, then we have 
	 \begin{align}
	 \sgn(\delta_{a})\sum_{i=1}^{n}\ell'(-y_{i}f(x_{i};\bm{\theta}^{*}))(-y_{i})\exp({\bm{w}^{*}}^{\top}x_{i}+b^{*})&\left[\frac{\exp(\varepsilon\bm{u}^{\top}x_{i})-\sum_{j=0}^{s-1}\frac{(\varepsilon \bm{u}^{\top}x_{i})^{j}}{j!}}{\varepsilon^{s}}\right]\notag\\
	 &+C(\tilde{\bm{\theta}}^{*},\mathcal{D})\varepsilon^{(k+1)/2-s}\ge 0\label{lemma4::eq-4}
	 \end{align}
	 When $ s\le\lfloor k/2\rfloor$ or, equivalently, $(k+1)/2-s\ge1/2$, taking the limit on the both sides of Eq.~\ref{lemma::eq-4}, we obtain that the inequality
	 \begin{equation}\label{lemma4::eq-5}
	\sgn(\delta_{a})\sum_{i=1}^{n}\ell'(-y_{i}f(x_{i};\bm{\theta}^{*}))(-y_{i})\exp({\bm{w}^{*}}^{\top}x_{i}+b^{*})(\bm{u}^{\top}x_{i})^{s} \ge 0,
	\end{equation}
	holds for every $\sgn(\delta_{a})\in\{-1, 1\}$ and this further implies 
	\begin{equation}\label{lemma4::eq-6}
	\sum_{i=1}^{n}\ell'(-y_{i}f(x_{i};\bm{\theta}^{*}))(-y_{i})\exp({\bm{w}^{*}}^{\top}x_{i}+b^{*})(\bm{u}^{\top}x_{i})^{s} = 0
	\end{equation} 
	Thus Eq.~\eqref{lemma4::eq-6} finishes our induction. 
\end{proof}
	
\subsection{Proof  of Proposition~\ref{thm::stationary}}
\begin{proof}
For every dataset $\mathcal{D}$ satisfying Assumption~\ref{assump::realizability}, by the Lagrangian interpolating polynomial, there always exists a polynomial of degree $p$ (i.e., $P(x;p)=\sum_{j}c_{j}\pi_{j}(x)$) defined on $\mathbb{R}^{d}$ such that it can correctly classify all samples in the dataset with margin at least one, i.e., $y_{i}P(x_{i};p)\ge 1,\forall i\in[n]$, where $\pi_{j}$ denotes the $j$-th monomial in the polynomial $P(x)$. 
Therefore, from Lemma~\ref{lemma4::a=0} and \ref{lemma::tensor}, it follows that 
\begin{align*}
\sum_{i=1}^{n}\ell'(-y_{i}f(x_{i};\bm{\theta}^{*}))e^{{\bm{w}^{*}}^{\top}x_{i}+b^{*}}y_{i}P(x_{i})=\sum_{j}c_{j}\sum_{i=1}^{n}\ell'(-y_{i}f(x_{i};\bm{\theta}^{*}))y_{i}e^{{\bm{w}^{*}}^{\top}x_{i}+b^{*}}\pi_{j}(x_{i})=0,
\end{align*}
where the second equality holds only when the degree of monomial $\pi_{j}$ is not larger than $\lfloor k/2\rfloor$. In other words, $p\le \lfloor k/2\rfloor$ or, equivalently, $k\ge 2p$, which is guaranteed by the assumption in Proposition~\ref{thm::stationary}. 
Since $y_{i}P(x_{i})\ge 1$ and $e^{{\bm{w}^{*}}^{\top}x_{i}+b^{*}}>0$ hold for  $\forall i\in[n]$ and the loss function $\ell$ is a non-decreasing function, i.e., $\ell'(z)\ge 0,\forall z\in\mathbb{R}$, then $\ell'(-y_{i}f(x_{i};\bm{\theta}^{*}))=0$ holds for all $i\in[n]$. In addition, from the assumption that every critical point of the loss function $\ell$ is a global minimum, it follows that $z_{i}=-y_{i}f(x_{i};\bm{\theta}^{*})$ achieves the global minimum of the loss function $\ell$ and this further indicates that $\bm{\theta}^{*}$ is a global minimum of the empirical loss $L_{n}(\bm{\theta})$. Furthermore, since at $k$-th order stationary point, the exponential neuron is inactive, $a^{*}=0$, then the set of parameters $\tilde{\bm{\theta}}^{*}$ is a global minimum of the loss function $\tilde{L}_{n}(\tilde{\bm{\theta}})$. Finally, since every critical point of the loss function $\ell(z)$ satisfies $z<0$, then for every sample, $\ell'(-y_{i}f(x_{i};\bm{\theta}^{*}))=0$ indicates that $y_{i}f(x_{i};\bm{\theta}^{*})>0$, or, equivalently, $y_{i}=\sgn(f(x_{i};\bm{\theta}^{*}))$. Therefore, the set of parameters $\bm{\theta}^{*}$ also minimizes the training error. In summary, the set of parameters $\tilde{\bm{\theta}}^{*}=(\bm{\theta}^{*},a^{*}, \bm{w}^{*}, b^{*})$ minimizes the  loss function $\tilde{L}_{n}(\tilde{\bm{\theta}})$ and the set of parameters $\bm{\theta}^{*}$ simultaneously minimizes the empirical loss function $L_{n}(\bm{\theta})$ and the training error $R_{n}(\bm{\theta};f)$. 
\end{proof}

\end{appendix}

\end{document}